\documentclass{article}

\usepackage{iclr2027_conference,times}
\usepackage{microtype}
\usepackage{graphicx}
\usepackage{booktabs}
\usepackage{url}
\usepackage{amsmath,amssymb,amsthm,mathtools,bm}
\usepackage{algorithm}
\usepackage{algorithmic}
\usepackage{longtable,array,multirow}
\usepackage{enumitem}
\usepackage{xcolor}
\usepackage{hyperref}
\hypersetup{hidelinks}

\usepackage{tabularx}
\usepackage{tikz}
\usetikzlibrary{arrows.meta,positioning}

\theoremstyle{plain}
\newtheorem{theorem}{Theorem}[section]
\newtheorem{lemma}[theorem]{Lemma}
\newtheorem{proposition}[theorem]{Proposition}
\newtheorem{corollary}[theorem]{Corollary}
\theoremstyle{definition}

\newtheorem{definition}[theorem]{Definition}
\newtheorem{protocol}[theorem]{Protocol}

\theoremstyle{remark}
\newtheorem{remark}[theorem]{Remark}

\newcommand{\E}{\mathbb{E}}
\newcommand{\Prob}{\mathbb{P}}
\newcommand{\R}{\mathbb{R}}
\newcommand{\1}{\mathbf{1}}

\newcommand{\cI}{\mathcal{I}}
\newcommand{\cL}{\mathcal{L}}

\newcommand{\TV}{\operatorname{TV}}
\newcommand{\clip}{\operatorname{clip}}
\newcommand{\argmin}{\operatorname*{arg\,min}}

\newcommand{\obs}{\mathrm{obs}}
\newcommand{\can}{\mathrm{canon}}
\newcommand{\STAR}{\mathrm{STAR}}
\newcommand{\adm}{\mathrm{adm}}

\ifdefined\STARredline

\newcommand{\revcolor}{\color{red}}
\else

\newcommand{\revcolor}{\color{black}}
\fi

\title{STAR-GRPO: Canonical Anchoring and Reliability-First Advantages against Representation-Dependent Reward Hacking}

\author{
Wan Tian\textsuperscript{1*} \ \ 
Zhongyi Li\textsuperscript{2*}\ \ 
Xiang Xu\textsuperscript{2}\ \ 
Minhao Zou\textsuperscript{1}\ \ 
Yijie Peng\textsuperscript{3$\dagger$} \ \ 
Fuzhen Zhuang\textsuperscript{2$\dagger$} \\[1ex]
\textsuperscript{1}Peking University \quad
\textsuperscript{2}Beihang University \quad
\textsuperscript{3}Nanjing University \\[0.5ex]
\textsuperscript{*}These authors contributed equally to this work. \quad
\textsuperscript{$\dagger$}Corresponding authors. \\
\texttt{Correspondence: pengyijie@nju.edu.cn; zhuangfuzhen@buaa.edu.cn}
}
\iclrfinalcopy

\begin{document}
\maketitle

{\revcolor
\begin{abstract}
Reward hacking occurs when policy optimization exploits a brittle reward interface or an overly permissive proxy objective, improving the training score without improving the underlying response quality. This phenomenon is amplified in group-relative policy optimization: an unsupported reward can shift the group baseline and alter the updates of other rollouts, while post-hoc or purely relative weighting cannot represent group-wide uncertainty. We propose \emph{Self-Tuned Anchored Reliability Group-Relative Policy Optimization} (STAR-GRPO), a reliability-first advantage estimator based on paired assessments of the same rollout. STAR separates the quality signal from its learning influence: score disagreement determines rollout reliability, relative reliability enters a self-tuned robust location--scale fit before group normalization, and absolute group reliability attenuates the resulting bounded advantage. The analysis establishes coordinate and second-moment bounds, characterizes exact centering through the weighted location equation, and gives reliability-dependent attenuation guarantees for outlying rewards. We evaluate STAR-GRPO in two complementary reward-hacking regimes. In token-interface exploitation, STAR prevents runaway optimization of the deployed-interface score while improving the canonical quality signal. In rubric-proxy overoptimization for medical reasoning, STAR improves independent semantic evaluation, narrows the proxy--judge discrepancy, and reduces overclaim while optimizing the same task proxy. Together, these results show that reliability-first normalization offers a principled way to limit unsupported reward influence on both group baselines and policy updates, while retaining the task reward as the optimization target.

\end{abstract}

\section{Introduction}
\label{sec:introduction}

Reinforcement learning from human feedback optimizes reward proxies for human preferences \citep{christiano2017deep,ouyang2022training,bai2022training}. As optimization proceeds, a policy can improve the proxy without a corresponding improvement in external quality \citep{skalse2022gaming,gao2023scaling}. An auxiliary assessment of the same response can reveal disagreement with the training score. The optimization question is then how to use this information: \emph{how should reliability enter group-relative advantages so that a suspicious reward has limited influence on both its own update and the baseline assigned to other responses?}

Ordinary group-relative policy optimization (GRPO) centers and scales rewards within each prompt group \citep{shao2024deepseekmath}. This coupling creates two distinct obstacles. First, an unreliable high reward changes the group mean and scale before any post-hoc weight is applied; reducing its own advantage does not undo the advantages already assigned to the other rollouts. Second, relative weighting alone cannot express common low confidence: multiplying every weight by the same constant leaves normalized weights unchanged. An advantage construction should therefore use reliability inside normalization and preserve its absolute magnitude in the update.

STAR-GRPO implements these two requirements while separating \emph{what to optimize} from \emph{how strongly to update}. A prespecified quality path selects or combines the training and anchor scores. Their discrepancy supplies an optimization reliability coefficient. Relative weights enter a pseudo-Huber fit of prompt-specific locations and a minibatch-shared scale; a bounded score then receives an absolute group-reliability factor. The resulting advantage satisfies $|A_i|\le w_i$ for any finite fitted location and positive scale. Exact fitting is needed for centering, not for this magnitude control.

We instantiate this principle in two reward-hacking regimes with different semantics but the same optimization structure. In \emph{token-interface exploitation}, direct token-index mappings can obtain high reward even when the corresponding decoded response is not supported by the canonical text interface \citep{zhang2026tompa}; STAR pairs the deployed and canonical views and uses their discrepancy to regulate the update. In \emph{rubric-proxy overoptimization}, a rubric-conditioned training judge can reward surface-level criterion satisfaction more strongly than an independent semantic evaluator; STAR keeps the proxy as the optimization target while using a rubric-free semantic anchor to determine how strongly each rollout should influence learning. This separation between the score being optimized and the evidence supporting that score is the central design principle of STAR-GRPO.

Our contributions are threefold:
\begin{itemize}[leftmargin=1.25em]
\item We introduce \emph{reliability-first} group-relative advantages: rollout reliability enters the robust group fit before normalization, while an explicit absolute-reliability factor preserves group-wide confidence in the final update.
\item We establish deterministic coordinate and second-moment bounds, characterize exact centering through the weighted location equation, and derive reliability-to-gradient attenuation guarantees. A one-outlier analysis shows linear attenuation in the rollout reliability, in contrast to post-hoc weighting and ordinary weighted mean--standard-deviation normalization.
\item We demonstrate the same mechanism across two distinct reward-hacking regimes. STAR suppresses representation-dependent token-interface exploitation while improving canonical quality; in rubric-proxy medical reasoning, the independent-judge score rises from 0.2706 to 0.3174, the proxy--judge gap decreases from 0.2935 to 0.2318, and the overclaim fraction decreases from 0.2464 to 0.2204.
\end{itemize}

Figure~\ref{fig:framework} summarizes the core idea: paired scoring exposes unsupported reward, reliability enters group statistics before normalization, and the resulting bounded advantage controls both local and group-wide learning influence.

\begin{figure}[t]
\centering
\includegraphics[width=\linewidth]{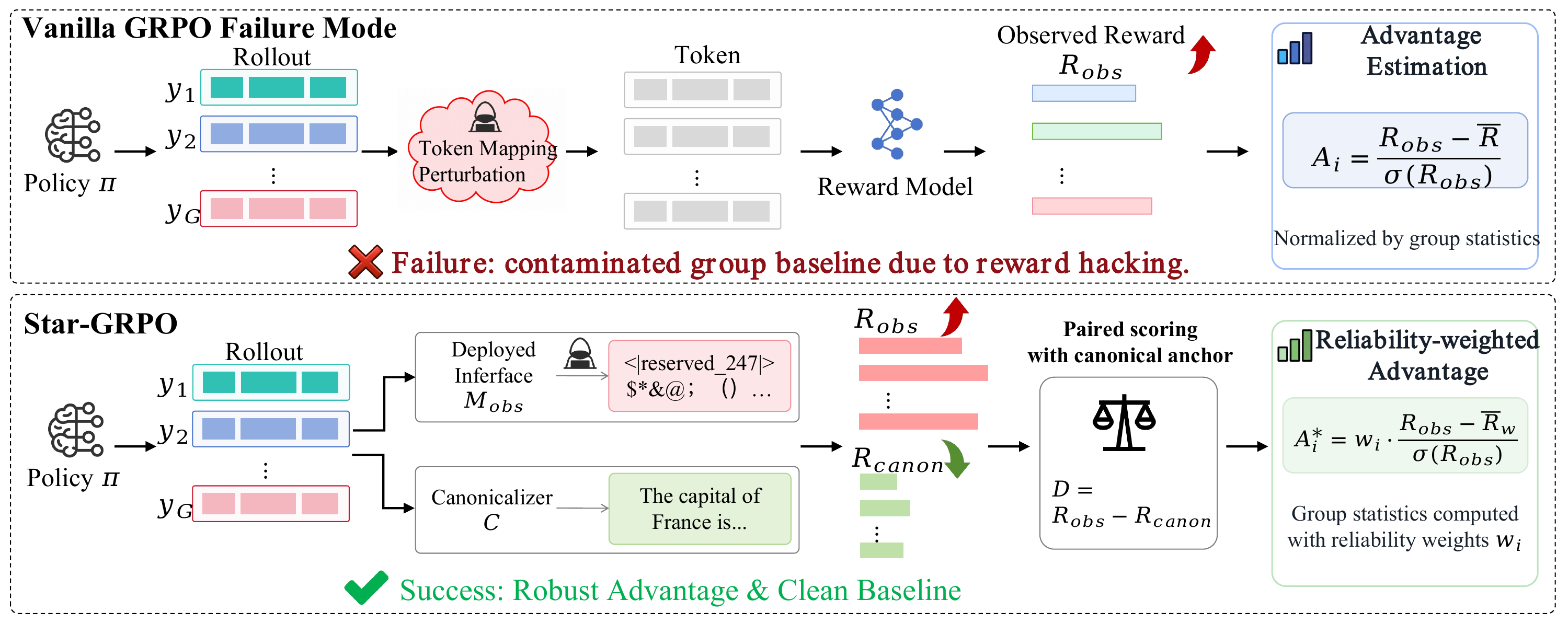}
\caption{Overview of token-interface reward hacking and STAR-GRPO. A deployed representation and a canonical rendering of the same rollout are scored in parallel; their discrepancy determines reliability, which enters group statistics \emph{before} normalization. The final STAR advantage combines a bounded robust score with rollout- and group-level reliability, preventing unsupported rewards from dominating either the baseline or the policy update.}
\label{fig:framework}
\end{figure}

\section{Related Work}
\label{sec:related-work-main}

Reward overoptimization and reward-model limitations motivate assessing progress beyond the training score \citep{gao2023scaling,casper2023open}. TOMPA studies token-interface attacks \citep{zhang2026tompa}; representation engineering supplies internal shortcut signals for modifying GRPO advantages \citep{wu2026rebound}. STAR instead derives reliability from paired reward assessments of each rollout.

Reward-model ensembles support conservative optimization \citep{coste2024ensembles} and uncertainty penalties \citep{zhai2024uprlhf}; robust reward training addresses preference artifacts \citep{liu2024rrm}. Conformal Feedback Alignment weights preference optimization using answer-level reliability \citep{chen2026cfa}. STAR's distinction is to place reliability inside the group fit and retain absolute attenuation afterward. Its robust estimation and calibration tools are established methods, detailed in Appendix~\ref{sec:related-work}.

Normalization and evaluation choices introduce further sources of bias. Dr.\ GRPO analyzes response-length and question-difficulty biases in GRPO \citep{liu2025understanding}, while DAPO studies token-level loss aggregation and training stability \citep{yu2025dapo}. These works motivate specifying the loss reduction separately from the advantages. LLM judges also exhibit position, verbosity, and self-enhancement biases \citep{zheng2023judging}. To separate the training proxy from evaluation, RQ2 uses a third judge that is never used in policy updates, providing an independent measurement of whether proxy optimization transfers to semantic quality.

\section{Problem Setup and Threat Model}
\label{sec:setup-main}

Both settings supply a pair of scores for the same rollout: a deployed or proxy score and an anchor score. We formalize the token-interface instantiation first, where the same reward model evaluates two representations and the discrepancy has a direct interface interpretation. RQ2 supplies a rubric proxy and a separate semantic anchor to the same advantage construction; its score reference and statistical scope are specified in Section~\ref{sec:experiment-rubric}.

For prompt \(X_b\), the old policy samples \(G\) rollouts
\(
O_{b,i}\sim\pi_{\theta_{\rm old}}(\cdot\mid X_b),  b\in[B],\quad i\in[G].
\)
The deployed mapping \(M_{\obs}\) and a frozen canonicalizer \(C\) produce
\begin{equation}
\begin{aligned}
R^{\obs}_{b,i}&=R_\phi(X_b,M_{\obs}(O_{b,i})), \quad
R^{\can}_{b,i}=R_\phi(X_b,C(O_{b,i})),\quad
D_{b,i}=R^{\obs}_{b,i}-R^{\can}_{b,i}.
\end{aligned}
\label{eq:intro-two-rewards}
\end{equation}
The canonicalizer decodes with the policy tokenizer, applies prespecified normalization and chat-template rules, and retokenizes with the reward tokenizer. Positive \(D\) measures deployed-interface reward unsupported by this anchor. Each rollout has a prespecified context \(H_{b,i}\in[H]\), such as tokenizer pair, task family, or response-length bin. Trusted benign outputs are divided before score inspection into discrepancy-fitting data \(\cI^D_{\rm fit}\) and conformal-calibration data \(\cI^D_{\rm cal}\). Policy training, attack development, independent clean audits, and final evaluation use disjoint identifiers. Context definitions, rare-context merges, and unseen-context fallback rules are frozen before calibration. Fix numerical regularizers \(\varepsilon_D\ge0\) and \(\varepsilon_w\ge0\); the positive scale constraints keep all displayed score denominators strictly positive. A fixed \(\kappa\in[0,1]\) determines how much deployed reward enters optimization:
\begin{equation}
R_{\kappa,b,i}=R^{\can}_{b,i}+\kappa D_{b,i},
\quad
r_{b,i}=h(R_{\kappa,b,i})\in[-R_{\max},R_{\max}],
\label{eq:anchored-quality}
\end{equation}
where \(h\) is fixed, monotone, bounded, and \(1\)-Lipschitz. Unless stated otherwise, STAR uses \(\kappa=1\); \(\kappa=0\) selects the canonical reward path. The quality path is part of the specified STAR instantiation. RQ1 additionally reports a canonical-only diagnostic control using the same bounded canonical quality path; all RQ1 reward-model evaluations use a 4,096-token input limit.

The threat model targets reward inflation that is expressed as disagreement between two views of the same rollout, including positive cross-interface shifts, heterogeneous clean discrepancy scales, heavy-tailed quality rewards before \(h\), and groups containing many attacked rollouts. The observed interface must be reachable by the deployed system; arbitrary mappings are reserved for audit stress tests. The following result formalizes why paired views supply information that is unavailable from a single scalar reward.

\begin{theorem}
\label{thm:main-identification}
(i) For any probability law \(P\) on \(\R\) and \(B_0>0\), there are a clean model and an attacked model with the same observable law \(R^{\obs}\sim P\), while the attacked canonical reward is smaller by \(B_0\). Hence no detector based only on \(R^{\obs}\) uniformly distinguishes the two models. (ii) If an attack adds the same random shift to every evaluated view, all pairwise reward differences are unchanged and every discrepancy-only test has type-I plus type-II error at least one.
\end{theorem}

The argument in Appendix~\ref{app:core-foundations} shows what a single scalar reward cannot reveal without additional assumptions. Paired views supply evidence about differential shifts; they do not by themselves establish that every observed discrepancy is an attack.

\section{Reliability-First Group-Relative Advantages}
\label{sec:method-main}

STAR first specifies the score pair and quality path, then estimates reliability, fits a robust group baseline, and forms policy advantages. We describe the split-calibrated interface instantiation below. RQ2 retains the advantage construction with a fixed discrepancy reference. Algorithm~\ref{alg:main-star} gives the procedure used for the formal analysis; experiment-specific solver and loss choices are documented separately.

\subsection{Trusted Discrepancy Calibration}

For trusted discrepancies \(D_{c,1:n_c}\) in context \(c\), choose \(z_{D,c}>0\) and \(0<s_{\min,c}<s_{\max,c}<\infty\). Following self-tuned robust estimation \citep{sun2024selftuned}, STAR fits a context location and scale with
\begin{align}
\ell_{n,z}(u,s)
&=\frac{ns}{z^2}
\left(\sqrt{1+\frac{z^2u^2}{ns^2}}-1\right)+\frac{s}{2},
\label{eq:main-ph-loss}\\
(\widehat m_{D,c},\widehat s_{D,c})
&\in
\argmin_{\substack{m\in\R\\s\in[s_{\min,c},s_{\max,c}]}}
\frac1{n_c}\sum_{j=1}^{n_c}\ell_{n_c,z_{D,c}}(D_{c,j}-m,s).
\label{eq:main-offline-fit}
\end{align}
The convex perspective loss jointly estimates location and influence scale, within the robust M-estimation framework \citep{huber1964robust}. Scale-boundary hits are logged as diagnostics. On the disjoint calibration split, define
\(
S_j^D=
\frac{D_j-\widehat m_{D,H_j}}
{\widehat s_{D,H_j}+\varepsilon_D}.
\)
For \(n_{\rm cal}=|\cI^D_{\rm cal}|\) and
\(k_\alpha=\lceil(n_{\rm cal}+1)(1-\alpha)\rceil\), let
\(\widehat q_{1-\alpha}\) be the \(k_\alpha\)th calibration order statistic, or \(+\infty\) if \(k_\alpha>n_{\rm cal}\). A rollout receives
\begin{equation}
S^D_{b,i}=
\frac{D_{b,i}-\widehat m_{D,H_{b,i}}}
{\widehat s_{D,H_{b,i}}+\varepsilon_D},
\quad
w_{b,i}=
\exp\{-\lambda[S^D_{b,i}-\widehat q_{1-\alpha}]_+\}.
\label{eq:main-weight}
\end{equation}
We take \(w=1\) when \(\widehat q_{1-\alpha}=+\infty\). The threshold controls when downweighting begins; \(\lambda>0\) controls its rate and is selected without reusing the conformal split. The coefficient $w$ controls optimization influence; it is not a posterior probability that the response is clean. With a fixed reference, as in RQ2, the same weight map is defined without the coverage claim of split calibration.

\subsection{Reliability-First Robust Advantages}

For group \(b\), compute total reliability $W_b$. A group with $W_b=0$ is skipped before evaluating the relative weights or effective size. For $W_b>0$, define
\begin{equation}
W_b=\sum_{i=1}^G w_{b,i},\quad
\bar w_b=\frac{W_b}{G},\quad
p_{b,i}=\frac{w_{b,i}}{W_b},\quad
G_{\mathrm{eff},b}=\frac{W_b^2}{\sum_iw_{b,i}^2}.
\label{eq:main-effective-size}
\end{equation}
The primary algorithm declares a group admissible when \(W_b\ge W_{\min}>0\) and \(G_{\mathrm{eff},b}\ge G_{\min}>1\). Non-admissible groups have zero reward advantage and do not enter the shared-scale fit. The recorded STAR runs use this group-admission rule. In the primary objective, keeping the policy-loss denominator equal to the original number of groups makes abstention reduce, rather than renormalize, the reward update. If no group is admissible, the optimizer step is skipped. The recorded STAR runs use token-mean loss reduction, specified in Appendix~\ref{app:experimental-design}.

Let \(\mathcal B_{\adm}\) be the admissible groups and \(B_{\adm}=|\mathcal B_{\adm}|\). For \(b\in\mathcal B_{\adm}\), one prompt-specific location \(\mu_b\) and a shared scale \(v\in[v_{\min},v_{\max}]\) minimize
\begin{align}
\ell_b^A(u,v)
&=
\frac{G_{\mathrm{eff},b}v}{z_A^2}
\left(\sqrt{1+\frac{z_A^2u^2}{G_{\mathrm{eff},b}v^2}}-1\right)
+\frac v2,\nonumber\\
\cL_A(\mu_{1:B},v)
&=
\frac1{B_{\adm}}
\sum_{b\in\mathcal B_{\adm}}\sum_{i=1}^G
p_{b,i}\ell_b^A(r_{b,i}-\mu_b,v).
\label{eq:main-online-objective}
\end{align}
Sharing \(v\) pools normalization information across the minibatch while retaining prompt-specific centers, which stabilizes scale estimation for small prompt groups. Absolute group reliability is then reintroduced explicitly in Equation~\eqref{eq:main-advantage}, so the strength of each group update remains coupled to the amount of supported reward evidence. At a fitted solution \((\widehat\mu_{1:B},\widehat v)\), put
\begin{equation}
x_{b,i}=
\frac{z_A(r_{b,i}-\widehat\mu_b)}
{\sqrt{G_{\mathrm{eff},b}}\widehat v},
\quad
\varphi_{b,i}=\frac{x_{b,i}}{\sqrt{1+x_{b,i}^2}},
\quad
\nu_b=\left(\frac1G\sum_iw_{b,i}^2\right)^{1/2},
\label{eq:main-score}
\end{equation}
and define the STAR advantage
\begin{equation}
A^{\STAR}_{b,i}
=
\bar w_b\,
\frac{w_{b,i}}{\nu_b+\varepsilon_w}\,
\varphi_{b,i}.
\label{eq:main-advantage}
\end{equation}
Placing reliability in the fit reduces a suspicious rollout's weight when estimating the baseline; a multiplier applied only after ordinary normalization would leave that baseline unchanged. The factor \(p_{b,i}\) provides this relative weighting. The factor \(\bar w_b\) preserves absolute reliability: for a fixed admitted set and \(\varepsilon_w=0\), multiplying all weights in a group by \(c\in(0,1]\) multiplies its advantages by \(c\). Thus relative weights determine which rollouts inform the baseline, while absolute reliability determines update strength. STAR requires two scoring calls per rollout: two interfaces to one model in RQ1, or two evaluators in RQ2. Each first-order or Newton pass costs \(O(BG)\); prompt locations parallelize, and the shared scale is one dimensional.

\begin{algorithm}[t]
\caption{STAR-GRPO with the reference numerical protocol}
\label{alg:main-star}
\begin{algorithmic}[1]
\REQUIRE Frozen fit/calibration objects; old and reference policies; \(G,\kappa,z_A,\lambda,W_{\min},G_{\min}\).
\STATE Sample \(G\) rollouts per prompt from the frozen old policy.
\STATE Evaluate deployed and canonical views; compute \(D\), \(r=h(R_\kappa)\), \(S^D\), and \(w\).
\STATE Compute \(W_b,\bar w_b,G_{\mathrm{eff},b}\); set reward advantages to zero for non-admissible groups.
\STATE If any group is admissible, fit \eqref{eq:main-online-objective} on those groups to the recorded KKT tolerance and compute \eqref{eq:main-advantage}.
\IF{no group is admissible or the solver/KKT check fails}
\STATE Skip the complete optimizer step and record the failure mode.
\ELSE
\STATE Stop gradients through all reward-side quantities and optimize the standard clipped GRPO objective with the original \(BG\) reduction; do not recenter, whiten, or rescale \(A^{\STAR}\).
\ENDIF
\STATE Log contexts, both rewards, \(D,w,\bar w,G_{\rm eff}\), abstention, solver residuals, boundary hits, and reward-gradient norms.
\end{algorithmic}
\end{algorithm}

\section{Theoretical Guarantees}
\label{sec:main-theory}

The construction admits three levels of analysis. Magnitude bounds follow directly from the advantage formula; centering additionally requires a location equation; statistical reliability requires assumptions about the paired assessments and reference data. All update bounds concern the initial reward-side direction at the old policy, with reward-side quantities held fixed. Proofs appear in Appendices~\ref{app:core-foundations}--\ref{app:advantages-and-optimization}.

\begin{theorem}[Magnitude control and centering]
\label{thm:main-geometry}
Fix weights $w_{b,i}\in[0,1]$ and an admissible set. For any finite locations and positive scale, Equation~\eqref{eq:main-advantage} satisfies
\[
\frac1G\sum_i(A^{\STAR}_{b,i})^2\le\bar w_b^2,
\qquad |A^{\STAR}_{b,i}|\le w_{b,i}.
\]
If the weighted location equation $\sum_i p_{b,i}\varphi_{b,i}=0$ holds, then $\sum_iA^{\STAR}_{b,i}=0$. Separately, objective~\eqref{eq:main-online-objective} is jointly convex and has a unique minimizer if every admitted prompt has two distinct positive-weight rewards.
Let $g^{\rm avg}_{b,i}=|O_{b,i}|^{-1}\sum_t\nabla_\theta\log\pi_{\theta_{\rm old}}(o_{b,i,t}\mid X_b,o_{b,i,<t})$ satisfy $\|g^{\rm avg}_{b,i}\|_2\le L_\pi$. Then
\[
\left\|U_b^{\STAR}\right\|_2
=\left\|\frac1G\sum_i A^{\STAR}_{b,i}g^{\rm avg}_{b,i}\right\|_2
\le\bar w_b L_\pi,
\qquad
\left\|\frac{A^{\STAR}_{b,i}g^{\rm avg}_{b,i}}G\right\|_2
\le\frac{L_\pi w_{b,i}}G.
\]
These direction bounds do not require an exact fit.
\end{theorem}

The proof uses $|\varphi|\le1$ and $\bar w_b\le\nu_b$. For an approximate fit, the centering error is exactly $|\sum_i A^{\STAR}_{b,i}|=\bar w_bW_be_b^\varphi/(\nu_b+\varepsilon_w)$, where $e_b^\varphi=|\sum_i p_{b,i}\varphi_{b,i}|$. This identity makes numerical centering directly auditable, while the coordinate and second-moment bounds hold for any finite fitted location and positive scale.

\begin{proposition}[One outlier: ordering and attenuation]
\label{prop:main-ordering}
Let $G\ge3$, $n=G-1$, $M\in(0,R_{\max}]$, $r=(0,\ldots,0,M)$, and $w=(1,\ldots,1,\epsilon)$ for $\epsilon\in(0,1]$. Use population-form standard deviations without a numerical stabilizer for the two comparators. Ordinary GRPO followed by multiplication by $w$ gives each clean rollout advantage $-1/\sqrt n$. Using the weighted mean $\mu_w=\sum_i w_ir_i/\sum_iw_i$ and weighted standard deviation $\sigma_w$ instead gives
\[
\widetilde A_i=w_i(r_i-\mu_w)/\sigma_w,
\qquad
\widetilde A_{\rm clean}=-\sqrt{\epsilon/n},
\qquad
\widetilde A_{\rm outlier}=\sqrt{n\epsilon}.
\]
For an admitted STAR group satisfying the exact location equation,
\[
|A^{\STAR}_{\rm outlier}|\le\epsilon,
\qquad |A^{\STAR}_{\rm clean}|\le\epsilon/n.
\]
If $W_{\min}<n$, $G_{\min}<n$, and the scale lies in fixed $[v_{\min},v_{\max}]\subset(0,\infty)$, exact STAR minimizers also satisfy $\widehat\mu_\epsilon\to0$ as $\epsilon\downarrow0$.
\end{proposition}

STAR therefore converts a rollout reliability of $\epsilon$ into an $O(\epsilon)$ coordinate bound for both the suspicious rollout and, through exact centering, the induced clean-rollout perturbation. This is precisely the reliability-first behavior sought by the method: unsupported mass vanishes before it can dominate group normalization.

\begin{theorem}[Conditional reliability control]
\label{thm:main-reliability}
Fix $\alpha\in(0,1)$ and condition on $\cI^D_{\rm fit}$. If clean calibration pairs and one new clean pair are exchangeable, then
\[
\Prob(w<1\mid Z=0,\cI^D_{\rm fit})\le\alpha,
\qquad \E[1-w\mid Z=0,\cI^D_{\rm fit}]\le\alpha.
\]
Here $Z=0$ denotes a clean sample and $Z=1$ an attacked sample. If $\Delta\ge0$, $\beta\in[0,1]$, and
$\Prob\{S^D\ge\widehat q_{1-\alpha}+\Delta\mid Z=1\}\ge1-\beta$, then
$\E[w\mid Z=1]\le\eta_A:=\beta+(1-\beta)e^{-\lambda\Delta}$.
\end{theorem}

The first claim is the standard finite-sample marginal split-calibration guarantee \citep{vovk2005algorithmic,angelopoulos2023gentle}. Corollary~\ref{cor:drift} extends this control to adaptive round $t$ under the stated conditional total-variation premise, giving a clean downweighting probability of at most $\alpha+\E\rho_t$. For separated attacked scores, the same weight map yields the exponential attenuation in Theorem~\ref{thm:main-reliability}.

\begin{corollary}[From reliability to initial influence]
\label{cor:main-attack-gradient}
Under the attacked-score condition of Theorem~\ref{thm:main-reliability} and the policy-score bound of Theorem~\ref{thm:main-geometry}, with skipped contributions set to zero,
\[
\E\!\left[\left\|A^{\STAR}_{b,i}g^{\rm avg}_{b,i}/G\right\|_2\mid Z_{b,i}=1\right]
\le L_\pi\eta_A/G.
\]
\end{corollary}

For the token-mean reduction used in the experiments, Corollary~\ref{cor:token-mean} gives the corresponding length-weighted control. For response lengths $L_i$ and $T=\sum_iL_i>0$, $\|U_{\rm token}\|_2\le L_\pi\sum_i L_iw_i/T$, directly linking the realized optimization reduction to the rollout reliabilities.

\section{Experiments}
\label{sec:experiments}

We evaluate STAR-GRPO in two qualitatively different reward-hacking regimes that share the same structural failure: the optimized score becomes larger than the support provided by an alternative view of the same rollout. \emph{RQ1} studies representation-dependent reward inflation caused by a token-interface mismatch. \emph{RQ2} studies semantic proxy overoptimization, where a rubric-conditioned training judge increasingly disagrees with an independent evaluator. The two settings test the same reliability-first principle with different score-pair semantics. Detailed model, calibration, optimizer, and evaluation configurations are given in Appendix~\ref{app:experimental-design}.

\subsection{RQ1: Mitigating Token-Interface Reward Hacking}
\label{sec:experiment-token}

\paragraph{Setup.}
We use Llama-3.2-1B-Instruct as the policy and Skywork-Reward-V2-Qwen3-8B \citep{liu2026skywork} as the reward model, with 10,000 WildChat training prompts \citep{zhao2024wildchat} and 100 curated NoveltyBench validation prompts \citep{zhang2025noveltybench}. Each validation prompt is sampled eight times. Training uses learning rate $10^{-6}$, batch size 64, eight rollouts per prompt, one PPO epoch, bfloat16, and a reference-policy KL coefficient of $0.001$. Prompt and response limits are 512 and 2,048 policy tokens, and all reward-model evaluations use a 4,096-token input limit.

TOMPA-GRPO optimizes the deployed-interface score $R^{\obs}$ through the identity token map $\Phi(j)=j$, which transfers response token IDs directly to the reward-model vocabulary. STAR evaluates the same rollout through both the deployed interface and a canonical decode--normalize--retokenize path, producing
\begin{equation}
D=R^{\obs}-R^{\can}.
\label{eq:token-experiment-discrepancy}
\end{equation}
The STAR quality path is canonical-primary, $r=10\tanh(R^{\can}/10)$ ($\kappa=0$), while $D$ determines how strongly each rollout is allowed to influence the robust group statistics and the final advantage. This directly targets representation-dependent reward inflation: optimization follows canonical quality, while the deployed/canonical mismatch controls reliability.

\begin{figure}[t]
\centering
\includegraphics[width=0.9\linewidth]{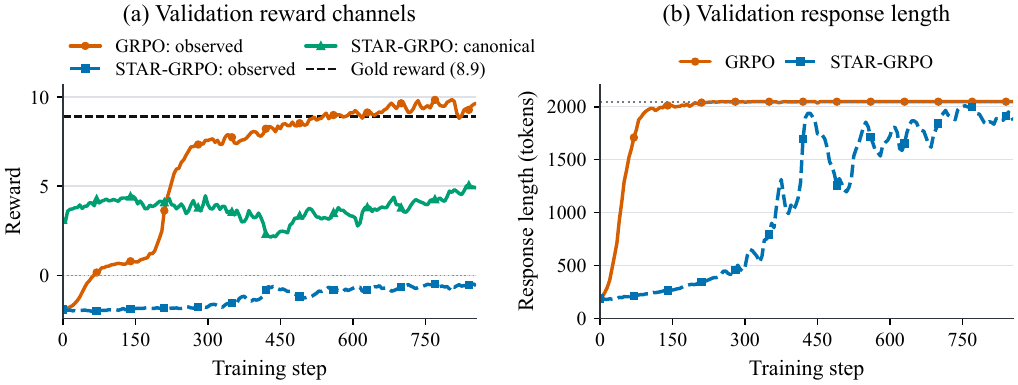}
\caption{RQ1 validation trajectories through step 855. (a) TOMPA-GRPO rapidly increases the deployed-interface reward, whereas STAR tracks both deployed and canonical views and optimizes the canonical quality path. The dashed 8.9 line is a reference-answer score on the observed scale. (b) TOMPA-GRPO reaches the 2,048-token response limit early in training, while STAR maintains a paired-score training signal throughout optimization.}
\label{fig:token-validation}
\end{figure}

\begin{figure}[t]
\centering
\includegraphics[width=\linewidth]{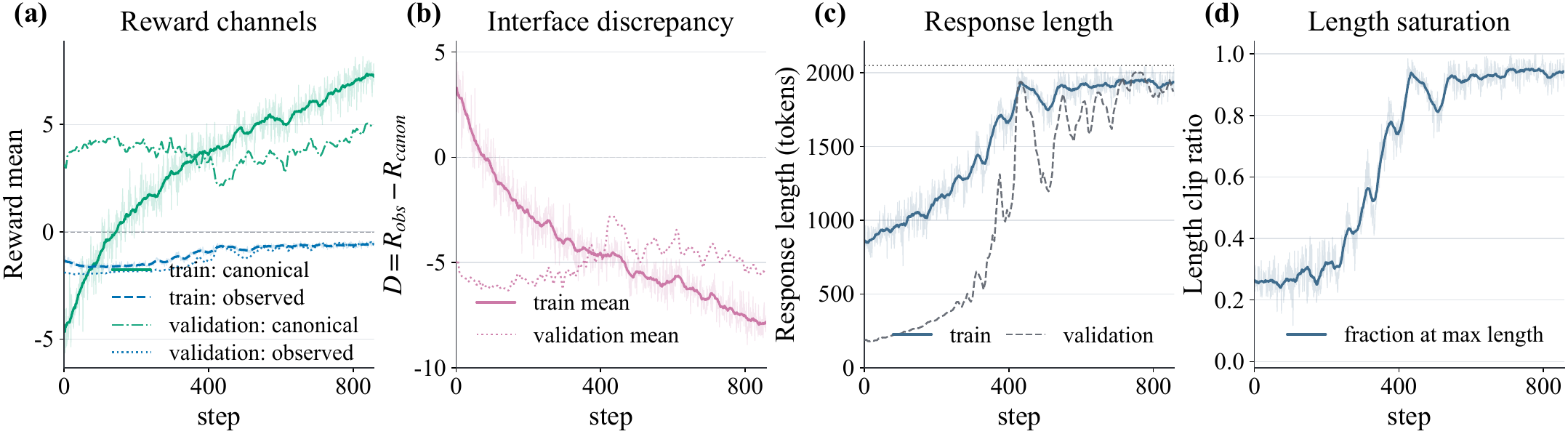}
\caption{Canonical-primary STAR in RQ1: (a) observed and canonical rewards, (b) interface discrepancy, (c) response length, and (d) training length-clip ratio. Training and validation traces follow the legend styles.}
\label{fig:token-training}
\end{figure}

\paragraph{Mitigating interface exploitation.}
TOMPA-GRPO exhibits the characteristic token-interface failure mode: its observed reward rises from $-1.894$ to 9.641, and the mean response length reaches the 2,048-token cap by step 235. In contrast, STAR improves the canonical validation reward from 3.121 to 4.902, an absolute gain of 1.781 (57.1\% relative), while the observed-interface reward ends at $-0.570$. The optimized signal therefore moves in the canonical direction rather than following the exploitable deployed-interface score. Figure~\ref{fig:token-validation} makes this separation visible over the full trajectory: the TOMPA reward channel accelerates toward the hacked high-score regime, whereas STAR preserves a canonical quality signal throughout training.

\paragraph{Selective reliability.}
STAR's average rollout reliability remains high, between 0.9847 and 1.0000, and the mean effective group size stays between 7.955 and 8.000 out of eight. At the same time, the minimum individual reliability reaches 0.105 over the recorded training trajectory (Figure~\ref{fig:token-training}). This is the intended operating regime: STAR leaves the bulk of supported rollouts nearly unchanged while strongly attenuating the small subset whose deployed score is poorly supported by the canonical view. The canonical-only diagnostic control and additional endpoint statistics are reported in Appendix~\ref{app:canonical-control}; they confirm that the canonical channel provides a meaningful anchor, while STAR turns that anchor disagreement into a general reliability mechanism that also applies when the proxy itself must remain the optimization target, as in RQ2.

\subsection{RQ2: Mitigating Rubric-Proxy Overoptimization}
\label{sec:experiment-rubric}

\paragraph{Setup.}
We train Qwen3-4B on RubricHub-Medical \citep{li2026rubrichub} and evaluate on HealthBench-Hard \citep{arora2025healthbench}. Both GRPO and STAR optimize the same GPT-4o-mini rubric-conditioned proxy score. STAR additionally obtains a rubric-free semantic assessment from Gemini-2.5-Flash-Lite and forms $D^{\rm rubric}=R^{\rm proxy}-R^{\rm anchor}$ to estimate reliability. The scalar quality path remains $r=R^{\rm proxy}$ ($\kappa=1$): the anchor does not replace the task reward; it determines how strongly the proxy score is trusted for learning. Claude-Sonnet-4-6 is used only as an independent validation judge and never enters the policy update.

Both methods use the same policy model, training data, 16 rollouts per prompt, PPO schedule, token-mean loss aggregation, and validation protocol. STAR uses fixed discrepancy center 0, scale 0.1, reliability threshold 1.645, and decay 2. Invalid score pairs are mapped to zero reliability and therefore excluded from the reward-side fit, consistent with the same reliability-first rule used for valid but weakly supported proxy scores.

\paragraph{Evaluation.}
The primary external metric is the independent-judge score. We additionally report the proxy--judge gap, criterion pass rate, and overclaim fraction. The proxy--judge gap measures the extent to which optimization progress on the training proxy is not reproduced by the independent evaluator; overclaim is the fraction of positive-weight rubric criteria credited by the proxy but not by the independent judge. Both validation judges score the same generated response against the same evaluation rubric, and Table~\ref{tab:rubric-results} uses the latest common checkpoint satisfying the fixed validation-availability criterion.

\begin{figure}[t]
\centering
\includegraphics[width=\linewidth]{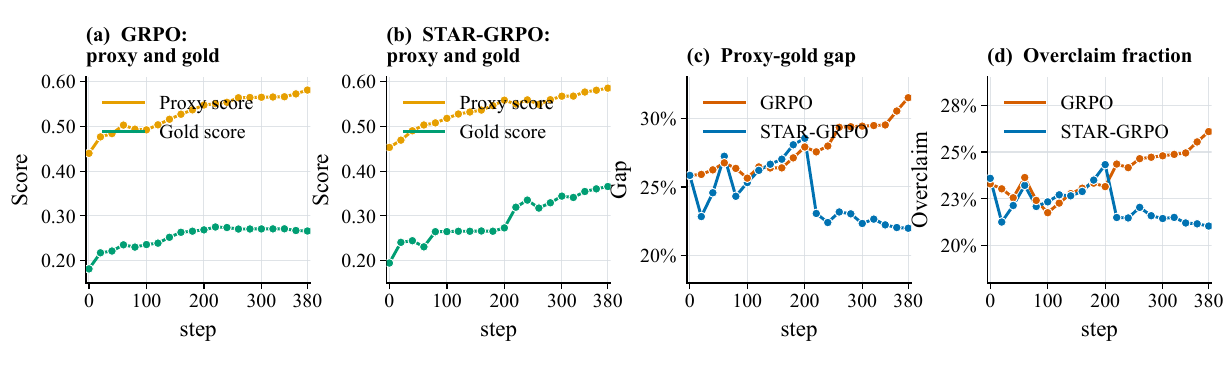}
\caption{RQ2 validation trajectories. (a,b) Training-proxy and independent-judge scores for GRPO and STAR-GRPO; (c) proxy--judge gap; (d) overclaim fraction. The independent judge is used only for evaluation and is distinct from STAR's rubric-free training anchor.}
\label{fig:rubric-main}
\end{figure}

\begin{figure}[t]
\centering
\includegraphics[width=\linewidth]{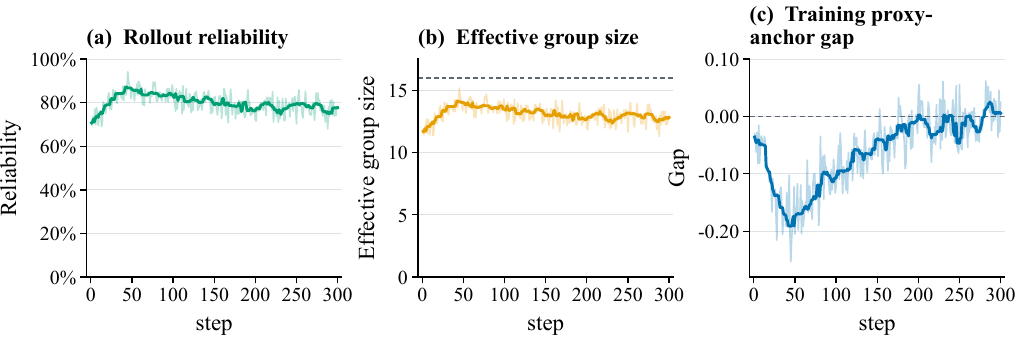}
\caption{STAR training diagnostics in RQ2: rollout reliability, effective group size, and proxy--anchor gap. Bold curves are 11-step rolling medians of per-update metadata. The training anchor differs from the independent evaluation judge.}
\label{fig:rubric-mechanism}
\end{figure}

\paragraph{Independent quality improves as proxy overclaim decreases.}
STAR reaches an independent-judge score of 0.3174 compared with 0.2706 for GRPO, corresponding to a 17.3\% relative increase. At the same checkpoint, the proxy--judge gap decreases from 0.2935 to 0.2318 (21.0\% relative reduction), while the overclaim fraction decreases from 0.2464 to 0.2204 (10.6\% relative reduction). Criterion pass rate simultaneously increases from 0.4752 to 0.5049. Importantly, these external gains occur while STAR's proxy score decreases from 0.5640 to 0.5492. This is the desired signature of reduced reward hacking: STAR is not simply driving the proxy harder; it converts more of the optimized proxy signal into quality that is reproduced by an independent evaluator.

\paragraph{Reliability is active throughout training.}
The 11-step rolling median rollout reliability remains approximately 0.70--0.90 and ends near 0.78, while mean effective group size ranges from 11.3 to 15.2 out of 16 (Figure~\ref{fig:rubric-mechanism}). The training proxy--anchor gap ranges from approximately $-0.253$ to $0.062$ and moves toward zero, and non-admissible groups reach approximately 10.9\%. These diagnostics show that the mechanism is active rather than degenerate: reliability continuously reshapes the contribution of individual rollouts and, when support becomes too weak, converts uncertainty into abstention instead of a potentially misleading reward update. The extra reward-side computation is lightweight relative to policy rollout: one anchor assessment per rollout plus an $O(BG)$ robust fit with a one-dimensional shared scale.

\paragraph{Two reward-hacking regimes, one optimization principle.}
RQ1 and RQ2 stress STAR in complementary conditions. RQ1 targets \emph{representation-dependent} reward inflation, where the deployed token interface itself creates an exploitable score channel; STAR shifts learning toward canonical quality and prevents the hacked observed reward from becoming the optimization target. RQ2 targets \emph{semantic proxy} overoptimization, where the proxy must remain the task reward; STAR therefore uses an independent semantic view only to modulate influence. Across both regimes, paired reward assessments provide the missing signal, and reliability-first normalization turns that signal into controlled group statistics and bounded policy advantages.

\section{Conclusion}
\label{sec:conclusion-main}

Reward hacking exposes a structural weakness of group-relative policy optimization: an unsupported reward can distort the group baseline and thereby change the advantages assigned to other rollouts. STAR-GRPO addresses this problem using paired assessments of each rollout to estimate reliability before normalization. Reliability shapes the robust group fit, while an absolute group-reliability factor attenuates the bounded advantages used for policy updates. The resulting construction provides coordinate and second-moment control, an auditable condition for exact centering, and reliability-dependent attenuation of reward-side influence. 

Across two distinct reward-hacking settings, the experiments show that STAR limits the influence of poorly supported scores. In token-interface exploitation, it suppresses runaway deployed-interface reward while improving canonical reward. In rubric-proxy medical reasoning, it improves independent-judge quality and reduces both proxy--judge disagreement and overclaim while retaining the proxy as the task reward. These results suggest that an auxiliary assessment can improve the robustness of group-relative learning by regulating how strongly a reward shapes the baseline and policy update, without requiring the same assessment to serve as the optimization target.


\emph{Ethics Statement.} STAR supports reliable reward-based training by monitoring paired scoring views. Deployment also calls for independent semantic evaluation, appropriate human oversight, and protection of sensitive logged outputs and model interfaces. Attack descriptions and evaluation artifacts should be released with safeguards proportionate to the systems they may affect.

\emph{Reproducibility Statement.} Equations~\eqref{eq:anchored-quality}--\eqref{eq:main-advantage} fully specify the STAR advantage construction. The appendix provides complete proofs, the reference numerical protocol, the paired-score interfaces, calibration and reliability parameters, optimizer settings, and the realized RQ1/RQ2 configurations needed to reproduce the reported experiments.

\emph{AI Use Statement.} Generative AI tools were used to assist with manuscript restructuring,
literature discovery, methodological critique, mathematical formulation and
checking, proof editing, and \LaTeX{} preparation. They were not used to
generate experimental observations or to independently determine empirical
conclusions. The authors retain responsibility for checking mathematical claims,
proofs, citations, and manuscript content against the stated assumptions,
derivations, experimental artifacts, and primary sources, and for the
final paper.
}

\clearpage
\appendix

{\revcolor

\section{Extended Related Work}
\label{sec:related-work}

Learning a reward from human preferences separates feedback collection from policy optimization \citep{christiano2017deep,ouyang2022training,bai2022training}. That separation creates a proxy objective whose maximization need not preserve the intended reward ordering. Formal work characterizes when proxy and true rewards can disagree under optimization \citep{skalse2022gaming}; scaling experiments quantify overoptimization under PPO and best-of-$n$ sampling \citep{gao2023scaling}. Underspecification explains why equivalent training-domain performance need not imply equivalent deployment behavior \citep{damour2022underspecification}, and surveys identify broader limitations of RLHF \citep{casper2023open}. Our identification results address the narrower information boundary of detecting shifts from paired scores.

Reward-model ensembles mitigate overoptimization through conservative objectives \citep{coste2024ensembles}, while diverse LoRA ensembles support uncertainty-penalized rewards \citep{zhai2024uprlhf}. Robust reward-model training uses a causal formulation and data augmentation to reduce prompt-independent preference artifacts \citep{liu2024rrm}. Conformal Feedback Alignment constructs answer-level reliability for DPO- and PPO-style alignment \citep{chen2026cfa}. These approaches motivate reliability-aware learning but intervene at different points. STAR's contribution is orthogonal: it uses a setting-specific score discrepancy to weight the group baseline \emph{before} normalization and then preserves the absolute reliability level in the final advantages, directly targeting contamination of group-relative statistics.

Group-relative optimization removes the learned critic by forming advantages within each prompt group \citep{shao2024deepseekmath}. Dr.\ GRPO identifies how per-response length reduction and reward-standard-deviation normalization change optimization weights \citep{liu2025understanding}. DAPO's token-level objective assigns equal weight to valid tokens, so longer responses contribute more to the aggregate direction \citep{yu2025dapo}. These observations motivate our explicit distinction between equal-response and token-mean reductions. STAR changes the reward-side advantage construction while retaining the specified PPO likelihood-ratio and clipping mechanics \citep{schulman2017ppo}; the initial-direction bounds do not imply guarantees for the full optimization trajectory.

Evaluation introduces a separate reliability problem. RewardBench tests preference discrimination across chat, reasoning, and safety cases \citep{lambert2025rewardbench}. LLM-as-a-judge studies document position, verbosity, and self-enhancement biases \citep{zheng2023judging}; length-controlled AlpacaEval explicitly adjusts for verbosity as a confounder \citep{dubois2024length}. These findings motivate interpreting response length as a diagnostic and paired disagreement as an operational signal. RQ2 uses RubricHub's structured training criteria \citep{li2026rubrichub} and the physician-designed HealthBench evaluation rubrics \citep{arora2025healthbench}, with our specified proxy and independent judges. Using these prompts and rubrics does not make our reported scores an evaluation under the original benchmarks' complete protocols. Similarly, RQ1 uses prompts from WildChat and NoveltyBench \citep{zhao2024wildchat,zhang2025noveltybench}, but reports reward and length dynamics rather than NoveltyBench's diversity metric.

The statistical tools serve distinct roles. Robust M-estimation limits the effect of extreme residuals \citep{huber1964robust}; bounded-influence and median-based estimators address heavy tails \citep{catoni2012challenging,lugosi2019mean,lecue2020robust}. Self-tuned pseudo-Huber estimation jointly adapts location and robustification scale \citep{sun2024selftuned}. Split conformal methods provide finite-sample marginal guarantees under exchangeability \citep{vovk2005algorithmic,angelopoulos2023gentle}. Conformal prediction beyond exchangeability studies coverage loss under distribution drift \citep{barber2023beyond}; our adaptive-round statement instead uses its explicitly stated conditional total-variation premise. Neither that premise nor clean exchangeability is supplied automatically by on-policy training.

\section{Core Proofs and Formal Foundations}
\label{app:core-foundations}

The following proof map links the main-text statements to their detailed derivations. We then fix notation and establish the identification results.

\subsection{Proof Map and Formal Setup}

The following proof map separates the main statements by the assumptions they use. The identification and calibration results concern observable score information; the magnitude and initial-direction bounds are deterministic conditional on the realized reward and weight arrays.

\begin{proof}[Proof of Theorem~\ref{thm:main-identification}]
Part (i) is Theorem~\ref{thm:single-view}. Part (ii) is Theorem~\ref{thm:view-consistent}, whose total-variation argument also covers randomized discrepancy-only tests after conditioning on their auxiliary randomness.
\end{proof}

\begin{proof}[Proof of Theorem~\ref{thm:main-reliability}]
The clean and separated-attack claims are Theorem~\ref{thm:weight-calibration}. The adaptive extension discussed after the theorem is Corollary~\ref{cor:drift}: its conditional total-variation premise transfers the static rank bound and averaging gives the outer marginal guarantee.
\end{proof}

\begin{proof}[Proof of Theorem~\ref{thm:main-geometry}]
The magnitude bounds and the conditional zero-sum statement follow from Theorem~\ref{thm:advantage-properties}. Theorem~\ref{thm:weighted-convexity} establishes convexity and uniqueness separately. Applying the magnitude bounds to the old-policy score vectors gives Theorem~\ref{thm:update-norm}; this step uses neither optimality of the fit nor exact centering.
\end{proof}

\begin{proof}[Proof of Corollary~\ref{cor:main-attack-gradient}]
A skipped rollout contributes zero. For an admissible rollout, Theorem~\ref{thm:update-norm} bounds its contribution by \(L_\pi w_{b,i}/G\). Taking the attacked conditional expectation and invoking \eqref{eq:attack-weight-attenuation} gives the displayed result.
\end{proof}

The proof of Proposition~\ref{prop:main-ordering} is given with the STAR advantage calculation in Section~\ref{sec:advantages}, where the limiting fit and the vanishing outlier contribution can be checked together. With these dependencies fixed, we now specify the rollout model, the trusted-data split, and the probability conventions used by every later derivation.

\label{sec:problem}

A minibatch contains prompts $X_1,\ldots,X_B$. For each prompt, the old policy samples $G$ rollouts
\[
O_{b,i}\sim\pi_{\theta_{\rm old}}(\cdot\mid X_b),
\quad b\in[B],\ i\in[G].
\]
The observed interface $M_{\obs}$ is the representation actually scored by the deployed pipeline. The canonical channel $C$ decodes the policy sequence into user-visible text, applies a prespecified normalization, and retokenizes with the reward tokenizer. Equation~\eqref{eq:intro-two-rewards} defines the two scores. Their difference is the minimal observable needed to separate quality from interface support, so we define it before specifying the trusted calibration model.

\begin{definition}[Representation discrepancy]
The representation discrepancy of rollout $(b,i)$ is
\[
D_{b,i}=R^{\obs}_{b,i}-R^{\can}_{b,i}.
\]
Positive discrepancy measures observed-interface reward unsupported by the canonical channel.
\end{definition}

Direct token-index mapping is not assumed to preserve semantics. It is an interface intervention. The online method uses the actual deployed interface as $M_{\obs}$; arbitrary counterfactual mappings are used for auditing unless they are reachable in deployment.

Let $H_{b,i}\in[H]$ be a finite prespecified context. A context may encode reward-model identity, policy--reward tokenizer pair, task class, language, response-length bin, or frozen policy epoch. Contexts are fixed before the held-out conformal calibration split is examined. Rare contexts follow a deterministic fallback hierarchy.

A trusted benign dataset is split into
\[
\cI^D_{\rm fit}\quad\text{and}\quad\cI^D_{\rm cal}.
\]
The first fits context-specific discrepancy locations and scales; the second calibrates a single pooled standardized quantile. Policy-training rollouts and final test examples are disjoint from both splits. Fix \(\varepsilon_D\ge0\) and \(\varepsilon_w\ge0\); positive scale constraints ensure every score denominator is strictly positive.

For clean outputs in context $c$, write
\begin{equation}
D=m_{D,c}+s_{D,c}\epsilon_{D,c},
\quad
\E\epsilon_{D,c}=0,
\quad
\E\epsilon_{D,c}^2=1,
\label{eq:clean-discrepancy-model}
\end{equation}
with $0<s_{D,c}<\infty$. Finite variance is used for the self-tuned scale interpretation, but conformal validity itself only requires exchangeability after the score map has been fitted.

For approximate cross-context pooling, let $F_{D,c}$ be the CDF of $\epsilon_{D,c}$ and assume
\begin{equation}
\sup_x|F_{D,c}(x)-F_D(x)|\le\eta_{\rm het}.
\label{eq:residual-heterogeneity}
\end{equation}
The special case $\eta_{\rm het}=0$ is a common standardized residual law.

A latent label $Z\in\{0,1\}$ is introduced only for detection-power analysis. We use three nested alternatives.
\begin{enumerate}
\item \emph{Standardized shift:}
\[
S^D\mid Z=1,H=c=\epsilon_{1,c}+\Delta_c,
\quad \Delta_c>0.
\]
\item \emph{Stochastic shift:} if $F_{0,c}$ and $F_{1,c}$ are clean and attacked standardized-score CDFs,
\[
F_{1,c}(t)\le F_{0,c}(t-\Delta_c).
\]
\item \emph{Quantile separation:}
\[
F^{-1}_{1,c}(\beta)>F^{-1}_{0,c}(1-\alpha).
\]
\end{enumerate}
Mean separation alone is not enough for power without variance or tail control.

The canonical view is an integrity anchor rather than the only permitted quality source. Fix \(\kappa\in[0,1]\) before validation and define
\begin{equation}
R_{\kappa,b,i}=R^{\can}_{b,i}+\kappa
(R^{\obs}_{b,i}-R^{\can}_{b,i}),
\quad
r_{b,i}=h(R_{\kappa,b,i}),
\label{eq:bounded-canonical-reward}
\end{equation}
where \(h:\R\to[-R_{\max},R_{\max}]\) is fixed, monotone, and \(1\)-Lipschitz. The default \(\kappa=1\) retains the deployed reward, whereas \(\kappa=0\) selects the canonical reward path. The value is fixed before calibration to specify the robustness--fidelity trade-off.

Let \(w_{b,i}\in(0,1]\) be the reliability weights defined below. Put
\begin{equation}
W_b=\sum_{i=1}^G w_{b,i},
\quad
\bar w_b=\frac{W_b}{G},
\quad
p_{b,i}=\frac{w_{b,i}}{W_b},
\quad
G_{\mathrm{eff},b}=\frac{W_b^2}{\sum_iw_{b,i}^2}.
\label{eq:effective-size}
\end{equation}
A group is admissible when \(W_b\ge W_{\min}>0\) and
\(G_{\mathrm{eff},b}\ge G_{\min}>1\). The primary fallback is skip: the group's reward advantage is zero, it is excluded from the shared-scale fit, and the policy loss retains the original group denominator. This is the fallback used in the realized experiments. If every group is non-admissible, the complete optimizer step is skipped.
The analysis below uses one final set of conventions to keep three sources of error separate.
\label{app:notation}

All conformal guarantees are conditional on the fitting split $\cI^D_{\rm fit}$ because the standardized score map is learned there. Unless explicitly stated otherwise, probabilities in Sections~\ref{sec:discrepancy-calibration}--\ref{sec:weights} integrate over the held-out calibration split and a new example. Conditional statements about STAR advantages treat rewards and reliability weights as fixed arrays.

For a generalized inverse, we use
\[
F^{-1}(u)=\inf\{x:F(x)\ge u\},
\quad u\in[0,1].
\]
The convention $F^{-1}(0)=-\infty$ and $F^{-1}(1)=\sup\{x:F(x)<1\}$ is used when needed. Ties in conformal scores are handled conservatively by strict rejection $S_{\rm new}>\widehat q$.

There are three distinct discrepancies in the analysis:
\begin{enumerate}
\item \emph{representation discrepancy} $D=R^{\obs}-R^{\can}$;
\item \emph{calibration estimation error} in $(\widehat m_{D,c},\widehat s_{D,c},\widehat q)$;
\item \emph{online weight error} $w-w^{\rm or}$ relative to latent clean indicators.
\end{enumerate}
The first is observed, the second is statistical, and the third is an oracle comparison. Human-utility error is not identified.

The theoretical group update uses one score vector $g_{b,i}$ per rollout. In token-level GRPO, this vector can be interpreted as the average of token score vectors. If every token score is bounded by $L_{\rm tok}$ and the rollout score is length-normalized, then $L_\pi=L_{\rm tok}$ is valid. If score vectors are summed rather than averaged, $L_\pi$ scales with response length and must be tracked explicitly.

\subsection{Identifiability Boundaries}
\label{app:identifiability-boundaries}

The identification results establish the information provided by an additional scoring view and characterize invariance under common shifts across views.

\label{sec:necessity}

A scalar reward cannot reveal a counterfactual score that was never evaluated. The next result formalizes this missing-information obstruction and motivates paired scoring.

\begin{theorem}
\label{thm:single-view}
Let $P$ be any probability law on $\R$ and let $B>0$. There exist two latent models with identical observable law $R^{\obs}\sim P$ such that in the first model every example is clean and in the second every example is representation-dependently attacked with canonical reward smaller by $B$. Consequently, no measurable detector based only on $R^{\obs}$ can uniformly distinguish the two model classes.
\end{theorem}

\begin{proof}
In Model A draw $Y\sim P$, set $Z=0$, $R^{\obs}=Y$, and $R^{\can}=Y$. In Model B draw the same $Y\sim P$, set $Z=1$, $R^{\obs}=Y$, and $R^{\can}=Y-B$. The marginal law of the only observed variable $R^{\obs}$ is $P$ in both models, but the attack labels and canonical rewards differ. Let $\delta:\R\to\{0,1\}$ be any detector. Its distribution is identical under the two models because its argument has the same law. Hence it cannot have both zero false-positive probability in Model A and zero false-negative probability in Model B. The same argument applies to randomized detectors after conditioning on their auxiliary randomness.
\end{proof}

\begin{remark}
Any scalar transformation of $R^{\obs}$, including clipping, a median, a median-of-means aggregate, Huberization, or mean--standard-deviation normalization, remains a function of the same non-identifying observation.
\end{remark}

Paired scoring resolves the single-view ambiguity only when an attack changes cross-view differences. The following result states the complementary failure boundary used throughout the later threat model.

\begin{theorem}
\label{thm:view-consistent}
Let $Q$ be any detector score measurable with respect to reward differences among a collection of views $(R_0,\ldots,R_K)$. If the clean and attacked laws of $Q$ are identical, every test based on $Q$ has type-I plus type-II error at least one. In particular, if an attack adds the same random shift $B$ to all views,
\[
R_k^{\rm attack}=R_k^{\rm clean}+B,
\quad k=0,\ldots,K,
\]
then all pairwise differences are unchanged and no discrepancy-only method can detect the attack.
\end{theorem}

\begin{proof}
Let $P_0$ and $P_1$ denote the clean and attacked laws of $Q$. For a test $\delta(Q)\in\{0,1\}$, the sum of errors is
\[
P_0\{\delta=1\}+P_1\{\delta=0\}
=1-\bigl(P_1\{\delta=1\}-P_0\{\delta=1\}\bigr)
\ge 1-\TV(P_0,P_1).
\]
If $P_0=P_1$, this lower bound equals one. Under a common additive shift across all views, every difference $R_k-R_\ell$ is unchanged pointwise, so any score measurable with respect to such differences has identical clean and attacked laws.
\end{proof}

Together, Theorems~\ref{thm:single-view} and \ref{thm:view-consistent} identify the exact scope of STAR. A canonical anchor supplies information absent from a single reward, but only attacks that create a detectable cross-interface discrepancy can be attenuated. The method is intentionally silent about shared semantic misspecification.

\section{Calibration and Reliability Guarantees}
\label{app:calibration-and-reliability}

Trusted paired examples define the context-specific discrepancy fit and the conformal threshold. The analysis connects these quantities to clean-sample preservation, attack attenuation, and distributional drift.

\subsection{Self-Tuned Context Calibration}

\subsubsection{Robust Context Fit}
\label{sec:discrepancy-calibration}

For context $c$, let $D_{c,1},\ldots,D_{c,n_c}$ be trusted fitting discrepancies. Choose a confidence parameter $z_{D,c}>0$ and scale bounds $0<s_{\min,c}<s_{\max,c}<\infty$. Define
\begin{equation}
\ell_{n,z}(u,s)
=
\frac{ns}{z^2}
\left(
\sqrt{1+\frac{z^2u^2}{ns^2}}-1
\right)+\frac{s}{2},
\label{eq:ph-loss}
\end{equation}
and
\begin{equation}
(\widehat m_{D,c},\widehat s_{D,c})
\in
\argmin_{m,\,s_{\min,c}\le s\le s_{\max,c}}
\frac1{n_c}\sum_{j=1}^{n_c}\ell_{n_c,z_{D,c}}(D_{c,j}-m,s).
\label{eq:offline-self-tuned}
\end{equation}
The coefficient $1/2$ makes the population scale approach the ordinary standard deviation in the large-sample regime under the self-tuned parameterization \citep{sun2024selftuned}.

Put
\[
a_{c,j}(m,s)=\frac{z_{D,c}^2(D_{c,j}-m)^2}{n_cs^2}.
\]
At an interior optimum, direct differentiation gives
\begin{align}
0
&=
\sum_{j=1}^{n_c}
\frac{D_{c,j}-\widehat m_{D,c}}
{\widehat s_{D,c}\sqrt{1+a_{c,j}(\widehat m_{D,c},\widehat s_{D,c})}},
\label{eq:offline-location-foc}
\\
0
&=
\frac1{n_c}\sum_{j=1}^{n_c}
\left[
\frac{n_c}{z_{D,c}^2}
\left(
\frac1{\sqrt{1+a_{c,j}(\widehat m_{D,c},\widehat s_{D,c})}}-1
\right)+\frac12
\right].
\label{eq:offline-scale-foc}
\end{align}
The location score has bounded magnitude:
\begin{equation}
\left|
\frac{u}{s\sqrt{1+z^2u^2/(ns^2)}}
\right|
\le \frac{\sqrt n}{z}.
\label{eq:bounded-offline-score}
\end{equation}

The offline fit must be globally auditable before its output can define a frozen conformal score map. The following proposition supplies that optimization property and the uniqueness condition used by the calibration protocol.

\begin{proposition}[Convexity of the offline fit]
For fixed $n,z>0$, the map $(m,s)\mapsto\ell_{n,z}(d-m,s)$ is jointly convex on $\R\times(0,\infty)$. Hence the empirical objective in \eqref{eq:offline-self-tuned} is jointly convex. If the positively weighted data contain at least two distinct values and the minimizer is interior, the objective is strictly convex and the minimizer is unique.
\end{proposition}

\begin{proof}
Define the convex scalar function
\[
\rho(t)=\frac{n}{z^2}\left(\sqrt{1+\frac{z^2t^2}{n}}-1\right).
\]
Its second derivative is
\[
\rho''(t)=\left(1+\frac{z^2t^2}{n}\right)^{-3/2}>0.
\]
The first term of \eqref{eq:ph-loss} is the perspective $s\rho(u/s)$, which is convex on $s>0$. Composition with the affine residual $u=d-m$ preserves convexity, and adding $s/2$ preserves convexity. Strict convexity follows from strict convexity of $\rho$ together with non-collinearity of residual directions generated by at least two distinct observations; equivalently, the sum of Hessian rank-one terms has full rank in $(m,s)$ at an interior point.
\end{proof}

The next result isolates the self-tuning mechanism without overclaiming a finite-sample variance estimate.

\begin{proposition}[Population scale adaptation]
Let $U$ be a nondegenerate mean-zero random variable with variance $\sigma^2<\infty$. Fix $n,z>0$ with $z^2<2n$. Suppose the population scale equation
\begin{equation}
\E\left[\frac1{\sqrt{1+z^2U^2/(ns^2)}}\right]
=1-\frac{z^2}{2n}
\label{eq:population-scale-equation}
\end{equation}
has an interior solution $s^\circ>0$. Let $\tau^\circ=\sqrt n\,s^\circ/z$ and
\[
\sigma_{\tau^\circ}^2
=
\E\left[U^2\1\{|U|\le\tau^\circ\}\right].
\]
Then
\begin{equation}
2\left(1-\frac1{\sqrt2}\right)\sigma_{\tau^\circ}^2
\le (s^\circ)^2
\le \sigma^2.
\label{eq:oracle-scale-bound}
\end{equation}
Consequently, if $\sigma_{\tau^\circ}^2\ge c_{\rm tv}\sigma^2$ for some $c_{\rm tv}>0$, then
\[
\sqrt{2(1-2^{-1/2})c_{\rm tv}}\,\sigma
\le s^\circ\le\sigma.
\]
\end{proposition}

\begin{proof}
Set $x=U^2/(\tau^\circ)^2$ and $g(x)=1-(1+x)^{-1/2}$. Equation~\eqref{eq:population-scale-equation} is
\[
\E g(x)=\frac{z^2}{2n}.
\]
For every $x\ge0$, concavity of the square root or direct differentiation yields $g(x)\le x/2$. Hence
\[
\frac{z^2}{2n}
\le
\frac{\E U^2}{2(\tau^\circ)^2}
=
\frac{z^2\sigma^2}{2n(s^\circ)^2},
\]
which gives $(s^\circ)^2\le\sigma^2$.

For $0\le x\le1$, the ratio $g(x)/x$ is decreasing and therefore at least $g(1)=1-2^{-1/2}$. Thus
\[
\frac{z^2}{2n}
=
\E g(x)
\ge
\left(1-\frac1{\sqrt2}\right)
\E\left[x\1\{x\le1\}\right]
=
\left(1-\frac1{\sqrt2}\right)
\frac{z^2\sigma_{\tau^\circ}^2}{n(s^\circ)^2}.
\]
Rearranging proves the lower bound.
\end{proof}

\begin{remark}
The lower bound depends on truncated variance. Uniform scale comparability therefore requires a regularity condition ensuring that a fixed truncation retains a constant fraction of total variance. Scale-boundary diagnostics complement this condition in practice.
\end{remark}

\subsubsection{Finite-Sample Calibration and Transfer}

The robust fit is not needed for rank validity, but its estimation error matters when pooled scores are interpreted across contexts. We therefore expose the required one-context deviation event and record the external result that can instantiate it.

Fix \(\delta_D\in(0,1)\). For every retained context \(c\), the fitting discrepancies are independent draws with location \(m_{D,c}\), finite positive standard deviation \(s_{D,c}\), and sample size \(n_c\ge5\) large enough for the source theorem below. Set
\[
z_{D,c}^2=\log(n_cH/\delta_D)<2n_c.
\]
After mapping \(v_0,V_0,\mu^\star,\sigma,\delta\) in \citet[Theorems~3.4--3.5, arXiv v5]{sun2024selftuned} to
\(s_{\min,c},s_{\max,c},m_{D,c},s_{D,c},\delta_D/H\), respectively, assume all of those theorems' scale-bound, truncated-variance, curvature, and sample-size premises hold. Equivalently for the present analysis, assume there are deterministic constants
\(C_{m,c}>0\) and \(0<L_{s,c}\le U_{s,c}<\infty\) such that
\[
\Prob\left(
|\widehat m_{D,c}-m_{D,c}|
>
C_{m,c}s_{D,c}
\sqrt{\frac{\log(n_cH/\delta_D)}{n_c}}
\ \text{or}\
\widehat s_{D,c}\notin[L_{s,c},U_{s,c}]
\right)
\le\frac{\delta_D}{H}.
\]
Rare contexts handled by the frozen fallback hierarchy are excluded from this simultaneous statement.

The role of the next result is to lift verified one-context guarantees to the finite frozen context collection used by STAR.

\begin{proposition}[Simultaneous context fitting]
Under the preceding contextwise deviation conditions, with probability at least
\(1-\delta_D\), simultaneously for all retained \(c\in[H]\),
\begin{align}
|\widehat m_{D,c}-m_{D,c}|
&\le
C_{m,c}s_{D,c}
\sqrt{\frac{\log(n_cH/\delta_D)}{n_c}},
\label{eq:context-location-bound}\\
L_{s,c}
&\le\widehat s_{D,c}\le U_{s,c}.
\label{eq:context-scale-bound}
\end{align}
If the instantiated base theorem gives constant-multiple scale bounds under a fixed-fraction truncated-variance condition, the same constants hold simultaneously here.
\end{proposition}

\begin{proof}
Apply the assumed one-context statement at failure probability
\(\delta_D/H\) and take a union bound over the retained contexts. The final sentence merely preserves any explicitly instantiated constant-multiple implication of the base theorem; it introduces no additional probability event.
\end{proof}

Fix \(\alpha\in(0,1)\). After fitting, compute for each
\(j\in\cI^D_{\rm cal}\)
\begin{equation}
S^D_j=
\frac{D_j-\widehat m_{D,H_j}}
{\widehat s_{D,H_j}+\varepsilon_D}.
\label{eq:calibration-score}
\end{equation}
Let \(n_{\rm cal}=|\cI^D_{\rm cal}|\ge1\) and
\(k_\alpha=\lceil(n_{\rm cal}+1)(1-\alpha)\rceil\).
If \(k_\alpha\le n_{\rm cal}\), let \(\widehat q_{1-\alpha}\) be the
\(k_\alpha\)th order statistic; otherwise set it to \(+\infty\).

The fitting split determines a fixed score map; the disjoint calibration split then supplies finite-sample rank validity independently of estimator consistency.

\begin{theorem}
\label{thm:conformal}
Conditional on \(\cI^D_{\rm fit}\), suppose the clean calibration pairs
\((D_j,H_j)\) and an independent new clean pair
\((D_{\rm new},H_{\rm new})\) are exchangeable. Then
\begin{equation}
\Prob\left(
S^D_{\rm new}>\widehat q_{1-\alpha}
\,\middle|\,\cI^D_{\rm fit}
\right)\le\alpha.
\label{eq:conformal-guarantee}
\end{equation}
The probability integrates over the calibration split and the new pair.
\end{theorem}

\begin{proof}
Condition on the fitting split. If \(k_\alpha=n_{\rm cal}+1\), then
\(\widehat q_{1-\alpha}=+\infty\) and the rejection event is empty.
Otherwise attach independent continuous auxiliary variables to the
\(n_{\rm cal}+1\) scores and rank the resulting pairs lexicographically.
Exchangeability makes the new pair's randomized rank uniform on
\(\{1,\ldots,n_{\rm cal}+1\}\). The strict event
\(S^D_{\rm new}>\widehat q_{1-\alpha}\) implies that this rank exceeds
\(k_\alpha\); ties can only remove strict rejections. Its probability is therefore at most
\[
\frac{n_{\rm cal}+1-k_\alpha}{n_{\rm cal}+1}\le\alpha.
\]
\end{proof}

Self-tuning improves the balance of pooled calibration when contexts differ mainly by location and scale, while exchangeability supplies conformal validity. The next result quantifies this transfer at a fixed threshold.

\begin{proposition}[Approximate context equalization]
Condition on \(\cI^D_{\rm fit}\) and let \((D_{\rm new},H_{\rm new})\) be an independent fresh clean pair. Suppose the clean residual CDF in context \(c\) satisfies \eqref{eq:residual-heterogeneity} and the reference CDF \(F_D\) has density bounded by \(M_D\). On the \(\cI^D_{\rm fit}\)-measurable event where
\[
\left|\frac{\widehat m_{D,c}-m_{D,c}}{s_{D,c}}\right|\le a_c,
\quad
\left|\frac{\widehat s_{D,c}}{s_{D,c}}-1\right|\le b_c<1,
\]
the following inequality holds simultaneously for every \(q\in\R\):
\begin{equation}
\left|
\Prob\!\left\{S^D_{\rm new}>q\,\middle|\,
\cI^D_{\rm fit},H_{\rm new}=c,Z_{\rm new}=0\right\}
-[1-F_D(q)]
\right|
\le
\eta_{\rm het}
+
M_D\left(
a_c+|q|b_c+\frac{|q|\varepsilon_D}{s_{D,c}}
\right).
\label{eq:context-equalization-bound}
\end{equation}
\end{proposition}

\begin{proof}
Condition throughout on the fitting split. Write \(D=m_{D,c}+s_{D,c}\epsilon_c\). The event \(S^D>q\) is equivalent to
\[
\epsilon_c>
\frac{\widehat m_{D,c}-m_{D,c}}{s_{D,c}}
+q\frac{\widehat s_{D,c}+\varepsilon_D}{s_{D,c}}.
\]
Relative to \(q\), the threshold displacement is at most
\(a_c+|q|b_c+|q|\varepsilon_D/s_{D,c}\).
The density bound converts this displacement into a reference-CDF difference, and \eqref{eq:residual-heterogeneity} contributes \(\eta_{\rm het}\).
\end{proof}

The static rank guarantee does not condition on an adaptively trained policy. To transfer it to round \(t\), the following corollary makes the feedback dependence explicit through a conditional distributional-stability premise.

\begin{corollary}
\label{cor:drift}
Let \(\mathcal G=\sigma(\cI^D_{\rm fit},\cI^D_{\rm cal})\) and let \(\mathcal H_t\supseteq\mathcal G\) be the pre-rollout training history. Let \(P_{\rm ref}^0(\cdot\mid\cI^D_{\rm fit})\) denote the law of an independent clean reference pair exchangeable with the calibration pairs. If, almost surely,
\[
\TV\!\left(
P_t^0(\cdot\mid\mathcal H_t),
P_{\rm ref}^0(\cdot\mid\cI^D_{\rm fit})
\right)\le\rho_t,
\]
then the outer marginal probability satisfies
\[
\Prob_t\{S^D>\widehat q_{1-\alpha}\}\le\alpha+\rho_t.
\]
If \(\rho_t\) is random, the right-hand side is replaced by
\(\alpha+\E\rho_t\).
\end{corollary}

\begin{proof}
For every realization of \(\mathcal H_t\), total-variation duality gives
\[
P_t^0(S^D>\widehat q_{1-\alpha}\mid\mathcal H_t)
\le
P_{\rm ref}^0(S^D>\widehat q_{1-\alpha}\mid\mathcal G)
+\rho_t.
\]
Average over the history. The reference draw is independent of the calibration split conditional on the fitting split, so Theorem~\ref{thm:conformal} bounds the first expectation by \(\alpha\). Averaging the drift term proves the claim.
\end{proof}
\subsubsection{Supporting Calculus and Boundary Conditions}
\label{app:discrepancy-calculus}

For completeness, the derivatives that underlie the convexity and numerical diagnostics can be written in closed form. For $u=d-m$, $a=z^2u^2/(ns^2)$, and
\[
\ell(u,s)=\frac{ns}{z^2}(\sqrt{1+a}-1)+\frac{s}{2},
\]
the derivatives are
\begin{align}
\partial_m\ell
&=-\frac{u}{s\sqrt{1+a}},
\label{eq:app-dm}
\\
\partial_s\ell
&=\frac{n}{z^2}\left(\frac1{\sqrt{1+a}}-1\right)+\frac12,
\label{eq:app-ds}
\\
\partial_{mm}^2\ell
&=\frac1{s(1+a)^{3/2}},
\label{eq:app-dmm}
\\
\partial_{ms}^2\ell
&=\frac{u}{s^2(1+a)^{3/2}},
\label{eq:app-dms}
\\
\partial_{ss}^2\ell
&=\frac{u^2}{s^3(1+a)^{3/2}}.
\label{eq:app-dss}
\end{align}
Therefore the Hessian equals
\begin{equation}
\nabla^2_{m,s}\ell
=
\frac1{s^3(1+a)^{3/2}}
\begin{bmatrix}
s^2 & su\\
su & u^2
\end{bmatrix}
=
\frac1{s^3(1+a)^{3/2}}
\begin{bmatrix}s\\u\end{bmatrix}
\begin{bmatrix}s&u\end{bmatrix},
\label{eq:hessian-rank-one}
\end{equation}
which is positive semidefinite and rank one for one observation. Two observations with distinct residuals yield linearly independent vectors $(s,u_j)$ and a positive-definite sum.

The same calculus clarifies when the population scale exists. For nondegenerate $U$, define
\[
H(s)=\E\left[(1+z^2U^2/(ns^2))^{-1/2}\right].
\]
The function is continuous and nondecreasing in $s$. If $\Prob(U=0)=0$, then $H(s)\to0$ as $s\downarrow0$ and $H(s)\to1$ as $s\to\infty$. Thus for $0<z^2/(2n)<1$, equation \eqref{eq:population-scale-equation} has a solution. If $U$ has an atom at zero, existence holds when $\Prob(U=0)<1-z^2/(2n)$. Strict monotonicity follows whenever $\Prob(U\ne0)>0$, yielding uniqueness.

For a large-sample interpretation, suppose $s_n^\circ$ solves \eqref{eq:population-scale-equation} with $z_n^2=o(n)$ and the family $s_n^\circ$ stays away from zero. Then $\tau_n^\circ=\sqrt n s_n^\circ/z_n\to\infty$, and dominated convergence gives $\sigma_{\tau_n^\circ}^2\to\sigma^2$. The population scale calculation above therefore sandwiches $(s_n^\circ)^2$ between a fixed multiple of a quantity converging to $\sigma^2$ and $\sigma^2$. The sharper conclusion $s_n^\circ\to\sigma$ follows from a first-order expansion of $g(x)$ and uniform integrability; this is the self-tuned interpretation established in the source robust-estimation theory.

At a finite scale boundary, the equality is replaced by a KKT condition. For empirical objective $L_c(m,s)$ on $[s_{\min,c},s_{\max,c}]$, the location equation always holds at an optimizer because $m$ is unconstrained, while
\begin{align*}
\partial_sL_c(\widehat m,\widehat s)&=0 &&\text{if }s_{\min,c}<\widehat s<s_{\max,c},\\
\partial_sL_c(\widehat m,s_{\min,c})&\ge0 &&\text{if }\widehat s=s_{\min,c},\\
\partial_sL_c(\widehat m,s_{\max,c})&\le0 &&\text{if }\widehat s=s_{\max,c}.
\end{align*}
Every experiment must report the proportion of contexts or minibatches hitting a boundary.

\subsection{Conformal Reliability, Separation, and Drift}

The calibration analysis has three components: a rank guarantee for a new clean pair, a tail-separation premise for attacked pairs, and a total-variation transfer condition for adaptive policy rounds.

\label{app:conformal-details}

To distinguish validity from power, first consider how accurately the empirical threshold localizes a population quantile. Let calibration scores be $S_1,\ldots,S_n$ and $S_{(1)}\le\cdots\le S_{(n)}$. Put $k=\lceil(n+1)(1-\alpha)\rceil$. If $k=n+1$, the finite-sample threshold is $+\infty$, so a nontrivial threshold requires $\alpha\ge1/(n+1)$ under this convention.

Let $\widehat F_n$ be the empirical CDF and suppose
\[
\sup_x|\widehat F_n(x)-F_0(x)|\le\epsilon_n.
\]
At $q=S_{(k)}$, $\widehat F_n(q)\ge k/n=u_k$, so $F_0(q)\ge u_k-\epsilon_n$, which implies
\[
q\ge F_0^{-1}((u_k-\epsilon_n)_+).
\]
For any $x<q$, at most $k-1$ observations are at or below $x$, so $\widehat F_n(x)\le(k-1)/n<u_k$. If $F_0(x)>u_k+\epsilon_n$, then $\widehat F_n(x)>u_k$, a contradiction. Under the strict-increase condition stated in Section~\ref{sec:weights}, this implies $q\le F_0^{-1}((u_k+\epsilon_n)\wedge1)$. DKW gives the uniform event with probability at least $1-2e^{-2n\epsilon_n^2}=1-\delta$.

If $F_0=\Phi$ and attacked standardized scores are $N(\Delta,1)$, then on the DKW event
\[
\mathrm{TPR}
\ge
1-\Phi\left(\Phi^{-1}((u_k+\epsilon_n)\wedge1)-\Delta\right).
\]
This formula is a power statement, not a validity statement. It describes how much standardized separation is needed after accounting for finite calibration size.

Suppose attacked scores have mean $\Delta$ and variance at most $\sigma_1^2$, and the threshold is deterministically bounded by $q_+<\Delta$. Cantelli's inequality gives
\[
\Prob(S\le q_+\mid Z=1)
=
\Prob(S-\Delta\le-(\Delta-q_+))
\le
\frac{\sigma_1^2}{\sigma_1^2+(\Delta-q_+)^2}.
\]
Thus TPR is at least $(\Delta-q_+)^2/[\sigma_1^2+(\Delta-q_+)^2]$. Under finite variance, the corresponding bound is polynomial rather than exponential.

The simultaneous context-equalization bound above holds simultaneously over thresholds on the fitting event. Hence, on the event
\(\{|\widehat q_{1-\alpha}|\le Q\}\), for a fresh clean pair independent of the calibration split conditional on \(\cI^D_{\rm fit}\), write \(\mathcal C_c\) for conditioning on
\((\cI^D_{\rm fit},\cI^D_{\rm cal},H_{\rm new}=c,Z_{\rm new}=0)\). Then
\[
\left|
\Prob\!\left(
S^D_{\rm new}>\widehat q_{1-\alpha}
\,\middle|\,\mathcal C_c
\right)
-
\left[1-F_D(\widehat q_{1-\alpha})\right]
\right|
\le
\eta_{\rm het}
+
M_D\left(
a_c+Qb_c+\frac{Q\varepsilon_D}{s_{D,c}}
\right).
\]
Relating the reference tail \(1-F_D(\widehat q_{1-\alpha})\) to \(\alpha\), or to the actual pooled clean tail, requires the empirical-quantile localization above and its failure probability; no exact group-conditional conformal guarantee is implied. We now convert this calibrated score into the rollout reliability used by STAR.
\label{sec:weights}

For a training rollout, calculate
\begin{equation}
S^D_{b,i}
=
\frac{D_{b,i}-\widehat m_{D,H_{b,i}}}
{\widehat s_{D,H_{b,i}}+\varepsilon_D}
\label{eq:training-discrepancy-score}
\end{equation}
and, for \(\lambda>0\), define
\begin{equation}
w_{b,i}
=
\exp\{-\lambda[S^D_{b,i}-\widehat q_{1-\alpha}]_+\}.
\label{eq:soft-weight}
\end{equation}
When \(\widehat q_{1-\alpha}=+\infty\), the convention is \(w_{b,i}=1\). The calibration split determines where downweighting begins; \(\lambda\) determines its post-threshold slope and is selected using separate validation data.

The basic shape of this map is used both in the attack attenuation theorem and in the local sensitivity calculation.

\begin{proposition}[Regularity of the soft reliability map]
For every rollout, \(0<w_{b,i}\le1\). If
\(S^D_{b,i}\le\widehat q_{1-\alpha}\), then \(w_{b,i}=1\).
The weight is nonincreasing and globally \(\lambda\)-Lipschitz in \(S^D_{b,i}\), and differentiable away from the threshold.
\end{proposition}

\begin{proof}
Below the threshold the map is constant. Above it,
\(w(s)=e^{-\lambda(s-\widehat q)}\), so
\(w'(s)=-\lambda w(s)\) and \(|w'(s)|\le\lambda\).
Continuity at the threshold proves global Lipschitzness.
\end{proof}

The conformal event can now be translated into a statement about the operational reliability weight. The separation condition is a tail condition, not a claim that a mean shift alone gives detection power.

\begin{theorem}
\label{thm:weight-calibration}
Under Theorem~\ref{thm:conformal}, for a new clean rollout,
\begin{equation}
\Prob(w<1\mid Z=0,\cI^D_{\rm fit})\le\alpha,
\quad
\E[1-w\mid Z=0,\cI^D_{\rm fit}]\le\alpha.
\label{eq:clean-weight-preservation}
\end{equation}
Let \(\Delta\ge0\) and \(\beta\in[0,1]\). If an attacked rollout satisfies
\[
\Prob\{S^D\ge\widehat q_{1-\alpha}+\Delta\mid Z=1\}\ge1-\beta,
\]
then
\begin{equation}
\E[w\mid Z=1]
\le
\beta+(1-\beta)e^{-\lambda\Delta}.
\label{eq:attack-weight-attenuation}
\end{equation}
\end{theorem}

\begin{proof}
The event \(w<1\) equals
\(S^D>\widehat q_{1-\alpha}\), and
\(0\le1-w\le\1\{w<1\}\).
For the attack statement, let
\(E=\{S^D\ge\widehat q_{1-\alpha}+\Delta\}\).
On \(E\), \(w\le e^{-\lambda\Delta}\); on \(E^c\), \(w\le1\).
Splitting the expectation proves the bound.
\end{proof}

\begin{remark}[Majority and fully attacked groups]
The discrepancy baseline is external to the current group, so a common positive shift is not redefined as normal. Relative weights alone would nevertheless cancel if all group members were equally suspicious. STAR addresses this second failure through the absolute factor \(\bar w_b\) and the primary \(W_{\min}\) abstention rule.
\end{remark}

For comparison, a hard-gating baseline is
\[
w^{\rm hard}_{b,i}=\1\{S^D_{b,i}\le\widehat q_{1-\alpha}\}.
\]
It provides an interpretable accepted set but creates discontinuous active-set changes. The primary STAR method uses soft individual weights together with deterministic group skipping at \(W_{\min}\) or \(G_{\min}\).

Quantitative power requires more than exchangeability. Conditional on the fitted score map, suppose the calibration scores are iid from a continuous CDF \(F_0\), the attacked test score is independent with CDF \(F_1\), and \(F_0\) is strictly increasing over the relevant quantile neighborhood. With \(u_k=k_\alpha/n_{\rm cal}\) and
\[
\epsilon_n=\sqrt{\frac{\log(2/\delta)}{2n_{\rm cal}}},
\]
when \(k_\alpha\le n_{\rm cal}\), the Dvoretzky--Kiefer--Wolfowitz localization derived above implies, with probability at least \(1-\delta\),
\begin{equation}
\Prob_{F_1}\{S^D>\widehat q_{1-\alpha}\}
\ge
1-F_1\!\left(F_0^{-1}((u_k+\epsilon_n)\wedge1)\right).
\label{eq:power-bound}
\end{equation}
Under stochastic shift \(F_1(t)\le F_0(t-\Delta)\), replace \(F_1\) on the right by \(F_0(\,\cdot-\Delta)\). This remains a power statement rather than a validity guarantee.

The same map supplies several implementation diagnostics.
\label{app:weight-analysis}

By the preceding Lipschitz property, if estimated discrepancy scores have error $|\widehat S-S^\star|\le\epsilon_S$, then
\[
|\widehat w-w^\star|\le\lambda\epsilon_S.
\]
This deterministic relation is useful in the local oracle comparison.

Calibration drift has an equally direct diagnostic implication. Under Corollary~\ref{cor:drift}, for deterministic \(\rho_t\),
\[
\Prob_t(w<1\mid Z=0)\le\alpha+\rho_t,
\quad
\E_t[1-w\mid Z=0]\le\alpha+\rho_t.
\]
For random \(\rho_t\), both right-hand sides are replaced by
\(\alpha+\E\rho_t\). If drift is not controlled, calibration diagnostics rather than nominal conformal levels must be used.

Reliability also determines the effective sample size. For nonnegative weights, $1\le G_{\rm eff}\le G$ whenever at least one weight is positive. The upper bound is Cauchy--Schwarz:
\[
(\sum_iw_i)^2\le G\sum_iw_i^2.
\]
The lower bound follows from $(\sum_iw_i)^2\ge\sum_iw_i^2$. Equality $G_{\rm eff}=G$ occurs for equal weights; $G_{\rm eff}=1$ occurs when only one weight is nonzero. Soft weights are strictly positive mathematically, but numerical underflow can make them zero, motivating stable log-weight computation.

Finally, the trusted external reference is essential. Suppose every discrepancy in a group is $D_i=M+\xi_i$, where $M$ is a large attack shift and $\xi_i$ has small spread. Any translation-equivariant group location estimator gives $\widehat m_D\approx M$, and standardized within-group discrepancies depend mainly on $\xi_i$. Hence weights based on $(D_i-\widehat m_D)/\widehat s_D$ approach clean-looking values even as $M\to\infty$. External trusted calibration removes this translation invariance with respect to the attacked group.

\section{Reliability-First Advantages and Optimization}
\label{app:advantages-and-optimization}

The weighted robust objective yields zero-sum bounded advantages and a controlled initial reward-side direction at the old policy.

\subsection{Weighted Robust Objective and Advantage Geometry}

Reliability enters the location--scale fit before normalization, and the leading group factor carries absolute reliability into the update. The following derivations establish convexity, KKT equations, zero sum, coordinate bounds, and common-weight scaling.

\subsubsection{Weighted Location--Scale Fit}
\label{sec:weighted-estimation}

GRPO typically uses few rollouts per prompt. Estimating an independent location and scale from $G\in\{4,8,16\}$ observations is unstable, especially after downweighting. STAR therefore assigns one robust location $\mu_b$ to each prompt and one scale $v$ to the minibatch. The model is descriptive rather than a claim that all prompts have identical population variance: the shared scale is an update normalization fitted from $BG$ weighted residuals.

Let \(\mathcal B_{\adm}\subseteq[B]\) denote the fixed admissible group set and \(B_{\adm}=|\mathcal B_{\adm}|>0\). For \(b\in\mathcal B_{\adm}\), define \(p_{b,i}\) and \(G_{\mathrm{eff},b}\) by \eqref{eq:effective-size}. Fix \(z_A>0\) and \(0<v_{\min}<v_{\max}\). The weighted objective is
\begin{align}
\cL_A(\mu_{\mathcal B_{\adm}},v)
=
\frac1{B_{\adm}}
\sum_{b\in\mathcal B_{\adm}}\sum_{i=1}^Gp_{b,i}
\Bigg[
&\frac{G_{\mathrm{eff},b}v}{z_A^2}
\left(
\sqrt{1+
\frac{z_A^2(r_{b,i}-\mu_b)^2}
{G_{\mathrm{eff},b}v^2}}
-1
\right)
+\frac v2
\Bigg].
\label{eq:weighted-online-objective}
\end{align}
Define
\begin{equation}
(\widehat\mu_{\mathcal B_{\adm}},\widehat v)
\in
\argmin_{\substack{\mu_b\in\R,\ b\in\mathcal B_{\adm}\\
v_{\min}\le v\le v_{\max}}}
\cL_A.
\label{eq:weighted-online-estimator}
\end{equation}
All reward, discrepancy, reliability, and fitted quantities are stop-gradient inputs to policy optimization.

Let
\[
a_{b,i}(\mu_b,v)
=
\frac{z_A^2(r_{b,i}-\mu_b)^2}
{G_{\mathrm{eff},b}v^2}.
\]
Because every \(\mu_b\) is unconstrained, every minimizer satisfies
\begin{equation}
\sum_{i=1}^Gp_{b,i}
\frac{r_{b,i}-\widehat\mu_b}
{\widehat v\sqrt{1+a_{b,i}(\widehat\mu_b,\widehat v)}}=0,
\quad b\in\mathcal B_{\adm}.
\label{eq:weighted-location-foc}
\end{equation}
If \(\widehat v\in(v_{\min},v_{\max})\), it also satisfies
\begin{equation}
\frac1{B_{\adm}}\sum_{b\in\mathcal B_{\adm}}
\left[
\frac{G_{\mathrm{eff},b}}{z_A^2}
\sum_{i=1}^Gp_{b,i}
\left(
\frac1{\sqrt{1+a_{b,i}(\widehat\mu_b,\widehat v)}}-1
\right)
+\frac12
\right]=0.
\label{eq:weighted-scale-foc}
\end{equation}
At a scale boundary, this equality is replaced by the corresponding KKT inequality.

The perspective form makes the fit globally auditable once the weights and active groups are frozen. The next result ensures that the solver targets a single statistical object rather than a local nonconvex surrogate.

\begin{theorem}
\label{thm:weighted-convexity}
Conditional on a fixed admissible set and fixed nonnegative weights with positive total mass in each admitted group, \eqref{eq:weighted-online-objective} is jointly convex in
\((\mu_{\mathcal B_{\adm}},v)\) on \(v>0\). If every admissible prompt has at least two distinct positive-weight quality rewards, the objective is strictly convex and the constrained minimizer is unique.
\end{theorem}

\begin{proof}
For each \((b,i)\), the bracketed term is
\(v\rho_b((r_{b,i}-\mu_b)/v)+v/2\), where
\[
\rho_b(t)=\frac{G_{\mathrm{eff},b}}{z_A^2}
\left(\sqrt{1+\frac{z_A^2t^2}{G_{\mathrm{eff},b}}}-1\right)
\]
is strictly convex. Its perspective is jointly convex. In the Hessian representation of Appendix~\ref{app:weighted-geometry}, two distinct residuals in each prompt identify its location coordinate and the shared scale direction, so the weighted Hessian sum is positive definite. Restriction to the convex scale interval preserves uniqueness.
\end{proof}

The reliability \(w\) is a function of \(R^{\obs}-R^{\can}\), while \(r=h(R_\kappa)\) generally depends on both views. The convexity and advantage results are deterministic conditional on observed \((r,w)\) and require no independence. Interpreting a fitted location as an unselected population parameter would require an additional weighted estimating-equation or cross-fitting condition; STAR does not identify it with human utility.

\subsubsection{Objective Geometry and Numerical Conditions}
\label{app:weighted-geometry}

For a fixed prompt $b$, set $n_b=G_{\rm eff,b}$, $u_{b,i}=r_{b,i}-\mu_b$, and $a_{b,i}=z_A^2u_{b,i}^2/(n_bv^2)$. The Hessian of one bracketed term in \eqref{eq:weighted-online-objective} with respect to $(\mu_b,v)$ is
\[
\frac1{v^3(1+a_{b,i})^{3/2}}
\begin{bmatrix}
v^2 & vu_{b,i}\\
vu_{b,i} & u_{b,i}^2
\end{bmatrix}.
\]
Multiplication by $p_{b,i}$ preserves positive semidefiniteness. In the full parameter vector \((\mu_b:b\in\mathcal B_{\adm},v)\), each observation contributes a rank-one vector supported on coordinate \(b\) and the shared-scale coordinate. Strict convexity requires enough distinct residuals in every prompt to identify each $\mu_b$ and aggregate variation to identify $v$.

At an interior optimum, \eqref{eq:weighted-scale-foc} can be rewritten as
\[
\frac1{B_{\adm}}\sum_{b\in\mathcal B_{\adm}}
\frac{G_{\rm eff,b}}{z_A^2}
\left[
1-
\sum_{i=1}^Gp_{b,i}
\frac1{\sqrt{1+a_{b,i}}}
\right]
=\frac12.
\]
A second-order expansion $1-(1+a)^{-1/2}\approx a/2$ gives
\[
\widehat v^2
\approx
\frac1{B_{\adm}}\sum_{b\in\mathcal B_{\adm}}\sum_{i=1}^Gp_{b,i}(r_{b,i}-\widehat\mu_b)^2.
\]
This is an interpretation, not an exact identity. Large residuals contribute subquadratically through the exact equation.

For fixed $v$, each location objective is strictly convex and its derivative is monotone, so bisection is globally safe. For fixed locations, the scale objective is convex on $v>0$ and one dimensional. Alternating exact minimization decreases the joint objective and converges to a global minimizer because the objective is convex and level sets are compact under the scale bounds and bounded rewards. A practical implementation may use a warm start from the previous minibatch.

When effective reliable mass is insufficient, the unique primary rule is \emph{skip}: set the group's reward advantages to zero, exclude it from the shared-scale fit, retain its entries in the loss tensor, and keep the original \(BG\) denominator. This is the equal-response reduction in the reference protocol. Both recorded STAR runs use zero reward advantages for skipped groups and retain their token entries in the token-mean denominator, as analyzed in Corollary~\ref{cor:token-mean}; this differs from the equal-response reference reduction.

\subsubsection{STAR Advantage and Deterministic Bounds}
\label{sec:advantages}

For an admissible group, define
\begin{equation}
\varphi_{b,i}
=
\frac{z_A}{\sqrt{G_{\mathrm{eff},b}}}
\frac{r_{b,i}-\widehat\mu_b}
{\widehat v\sqrt{1+z_A^2(r_{b,i}-\widehat\mu_b)^2/
(G_{\mathrm{eff},b}\widehat v^2)}}.
\label{eq:phi-definition}
\end{equation}
For small residuals the score is approximately linear, while large residuals saturate. The following elementary bound is what makes the subsequent geometry distribution free.

\begin{lemma}[Bounded self-tuned score]
For every admissible rollout, \(|\varphi_{b,i}|\le1\).
\end{lemma}

\begin{proof}
Let \(x=z_A(r_{b,i}-\widehat\mu_b)/
(\sqrt{G_{\mathrm{eff},b}}\widehat v)\).
Then \(\varphi=x/\sqrt{1+x^2}\).
\end{proof}

Let
\[
\nu_b=\left(\frac1G\sum_{j=1}^Gw_{b,j}^2\right)^{1/2}.
\]
The STAR advantage is
\begin{equation}
A^{\STAR}_{b,i}
=
\bar w_b
\frac{w_{b,i}}{\nu_b+\varepsilon_w}\,
\varphi_{b,i}.
\label{eq:star-advantage}
\end{equation}
Relative weights \(p_{b,i}\) act inside the fitted location and scale. The factor \(\bar w_b\) carries the group's absolute reliability to the policy update, closing the common-scale cancellation present when only \(w_i/\nu_b\) is used.

The next theorem distinguishes magnitude control from centering. Magnitude control uses the formula alone. At an exact fit, the unconstrained location coordinate satisfies its first-order equation even when the shared scale is at a boundary.

\begin{theorem}
\label{thm:advantage-properties}
Suppose group \(b\) is admissible, \(w_{b,i}\in[0,1]\), and \(\varepsilon_w\ge0\). Evaluate Equation~\eqref{eq:star-advantage} at any finite location and positive scale. Then:
\begin{enumerate}
\item \emph{Conditional zero sum:} if \(\sum_i p_{b,i}\varphi_{b,i}=0\), then \(\sum_iA^{\STAR}_{b,i}=0\). This additional condition is not needed for the remaining statements.
\item \emph{Empirical second moment:}
\[
\frac1G\sum_{i=1}^G(A^{\STAR}_{b,i})^2
\le
\bar w_b^2
\left(\frac{\nu_b}{\nu_b+\varepsilon_w}\right)^2
\le\bar w_b^2.
\]
\item \emph{Individual bound:}
\[
|A^{\STAR}_{b,i}|
\le
\bar w_b\frac{w_{b,i}}{\nu_b+\varepsilon_w}
\le w_{b,i}\le1.
\]
\item \emph{Clean reduction:} if every \(w_{b,i}=1\), then
\(\bar w_b=\nu_b=1\) and STAR reduces to the self-tuned anchored-quality score divided by \(1+\varepsilon_w\).
\item \emph{Separated attack attenuation:} if
\(w_{b,i}\le e^{-\lambda\Delta}\), then
\(|A^{\STAR}_{b,i}|\le e^{-\lambda\Delta}\).
\end{enumerate}
\end{theorem}

\begin{proof}
For the first statement only, the location-equation premise gives
\(\sum_ip_{b,i}\varphi_{b,i}=0\).
Since \(p_{b,i}=w_{b,i}/W_b\), multiplying by the common factor
\(\bar w_b/(\nu_b+\varepsilon_w)\) proves zero sum. The bounded-score calculation above gives
\[
\frac1G\sum_i(A^{\STAR}_{b,i})^2
\le
\bar w_b^2
\frac{G^{-1}\sum_iw_{b,i}^2}
{(\nu_b+\varepsilon_w)^2}.
\]
Finally, RMS dominates the arithmetic mean, so
\(\bar w_b\le\nu_b\). This proves the individual bound; the last two statements follow by substitution.
\end{proof}

For a numerical solve, let
\(e_b^\varphi=|\sum_i p_{b,i}\varphi_{b,i}|\) be the reported location-equation residual. Then
\[
\left|\sum_iA^{\STAR}_{b,i}\right|
=
\frac{\bar w_bW_b}{\nu_b+\varepsilon_w}\,e_b^\varphi.
\]
Thus exact zero sum requires an exact location equation, while the magnitude bounds continue to hold for arbitrary finite fitted locations and positive scale. The centering deviation can be audited by recording the location-equation residual; certifying the joint optimum also requires the scale KKT condition. These statements concern real arithmetic and do not certify a particular floating-point implementation. The primary numerical protocol requires a final safeguarded location solve at the fitted scale and residual checks before accepting a policy step. RQ2's recorded implementation and its more limited diagnostics are specified in Appendix~\ref{app:experiment-2-details}.

Uniformly shrinking every weight should weaken the group update rather than be removed by normalization. The group factor makes this behavior exact when the numerical floor is zero and further attenuates the update with a positive numerical floor.

\begin{proposition}[Common reliability scaling]
Fix the reward arrays and all other groups, and replace one group's weights \(w\) by \(cw\) for \(c\in(0,1]\), without changing the full admissible set. Use the same solution selection for the unchanged fitting objective. The relative weights, effective size, fitted parameters, and \(\varphi\) are unchanged. If \(\varepsilon_w=0\), then \(A^{\STAR}(cw)=cA^{\STAR}(w)\). If \(\varepsilon_w>0\), every coordinate is multiplied by
\[
\frac{c^2(\nu_b+\varepsilon_w)}
{c\nu_b+\varepsilon_w}\le c.
\]
\end{proposition}

\begin{proof}
Both \(p_i=w_i/W_b\) and \(G_{\rm eff}=W_b^2/\sum_iw_i^2\) are invariant to a common positive scale, while \(\bar w_b\) and \(\nu_b\) each scale by \(c\). Substitution in \eqref{eq:star-advantage} proves the claim.
\end{proof}

For the same quality vector \(r=(0,\ldots,0,M)\), ordinary GRPO has one advantage \(\sqrt{G-1}\) and \(G-1\) advantages \(-1/\sqrt{G-1}\), independently of \(M>0\). Multiplying only the outlier by a small reliability weight does not undo the negative clean signs already assigned by the contaminated normalization. Reliability-first fitting instead sends the outlier's mass to zero before solving the location equation.

\begin{proof}[Proof of Proposition~\ref{prop:main-ordering}]
Write $n=G-1$. The ordinary mean and population-form standard deviation of $(0,\ldots,0,M)$ are $M/G$ and $M\sqrt n/G$. Every clean standardized reward is therefore $-1/\sqrt n$, unchanged by post-hoc multiplication by its unit weight.

For the weighted comparator, define
\[
\mu_w=\frac{\sum_iw_ir_i}{n+\epsilon}
=\frac{\epsilon M}{n+\epsilon},\qquad
\sigma_w^2=\frac{\sum_iw_i(r_i-\mu_w)^2}{n+\epsilon}
=\frac{n\epsilon M^2}{(n+\epsilon)^2}.
\]
Substituting into $\widetilde A_i=w_i(r_i-\mu_w)/\sigma_w$ gives $\widetilde A_{\rm clean}=-\sqrt{\epsilon/n}$ and $\widetilde A_{\rm outlier}=\sqrt{n\epsilon}$. Thus weighting before normalization attenuates every coordinate even with the ordinary weighted mean and variance.

For STAR, Theorem~\ref{thm:advantage-properties} gives $|A_G^{\STAR}|\le\epsilon$ without an exact solve. With the location equation satisfied, all $n$ clean coordinates are equal and zero sum gives $nA_{\rm clean}^{\STAR}=-A_G^{\STAR}$. Hence $|A_{\rm clean}^{\STAR}|\le\epsilon/n$. This is a finite-$\epsilon$ bound for every admitted group, including when its scale is shared with other groups.

For the location limit, $W=n+\epsilon$ and $G_{\rm eff}=(n+\epsilon)^2/(n+\epsilon^2)\to n$, so the stated strict admission thresholds ensure eventual admission. The fitted location lies in $[0,M]$, and the scale lies in the fixed compact interval. Along any convergent subsequence of exact fits, the location equation is
\[
n\varphi(-\widehat\mu_\epsilon,\widehat v_\epsilon;G_{\rm eff})
+\epsilon\varphi(M-\widehat\mu_\epsilon,\widehat v_\epsilon;G_{\rm eff})=0,
\]
where $\varphi(u,v;n')=(z_Au/(\sqrt{n'}v))/\sqrt{1+z_A^2u^2/(n'v^2)}$. Continuity and $|\varphi|\le1$ imply $\varphi(-\mu_\star,v_\star;n)=0$ at the limit, hence $\mu_\star=0$. Compactness then gives $\widehat\mu_\epsilon\to0$ for the full sequence. The displayed advantage bounds already imply that all coordinates vanish.
\end{proof}

The comparison uses no standard-deviation stabilizer. Replacing $\sigma_w$ by $\sigma_w+\delta$ with a fixed $\delta>0$ gives $O(\epsilon)$ coordinates as $\epsilon\downarrow0$ for fixed $M$; its constants can depend on $M/\delta$. STAR's $|A_i|\le w_i$ is uniform over finite reward arrays and positive scales. The proposition is an algebraic design comparison, not an empirical ablation or a claim that every alternative robust estimator lacks such a bound.

\subsubsection{Algebraic Consequences and Design Variants}
\label{app:star-advantage}

Several short calculations clarify the role of numerical stabilization and the choice of denominator. The epsilon $\varepsilon_w$ is common within a prompt, so it preserves zero sum when the weighted location equation holds:
\[
\sum_iA^{\STAR}_{b,i}
=
\frac{\bar w_b}{\nu_b+\varepsilon_w}\sum_iw_{b,i}\varphi_{b,i}=0.
\]
It only shrinks the second moment relative to the ideal RMS normalization.

If \(h\) is the identity and anchored quality rewards are transformed as \(r'_{b,i}=a r_{b,i}+c_b\) with \(a>0\), the prompt locations transform exactly as \(a\widehat\mu_b+c_b\) and the shared scale as \(a\widehat v\) when both scale bounds are multiplied by \(a\). The normalized score \(\varphi\) and reliability factors are invariant. The bounded nonlinear map \(h\) intentionally breaks exact affine equivariance to control optimizer gradients.

The one-outlier post-hoc calculation has already been established in the proof of Proposition~\ref{prop:main-ordering}. A denominator comparison isolates the role of the leading absolute factor \(\bar w_b\). The mean-denominator ablation is
\[
A^{\rm mean}_{b,i}
=
\bar w_b\frac{w_{b,i}}{\bar w_b+\varepsilon_w}\varphi_{b,i},
\]
whereas STAR uses the RMS denominator \(\nu_b\). Because \(\nu_b\ge\bar w_b\), the RMS variant has no larger coordinate magnitude and yields the sharper energy bound in Theorem~\ref{thm:advantage-properties}. With one unit weight and the rest zero, the mean-denominator factor on the surviving score is one for \(\varepsilon_w=0\), while STAR's factor is \(1/\sqrt G\). The leading \(\bar w_b\), rather than the denominator choice alone, is what preserves common low reliability.

\subsection{Initial Policy Direction and Its Scope}

The advantage bounds below apply to the initial reward-side direction at the old policy. The clipped policy objective is stated first to identify this direction within the training update.

\subsubsection{Policy Objective and Old-Policy Guarantee}
\label{sec:star-grpo}

For token \(t\) of rollout \((b,i)\), let
\[
\rho_{b,i,t}(\theta)=
\frac{\pi_\theta(o_{b,i,t}\mid X_b,o_{b,i,<t})}
{\pi_{\theta_{\rm old}}(o_{b,i,t}\mid X_b,o_{b,i,<t})},
\quad L_{b,i}=|O_{b,i}|.
\]
Define the clipped token term
\[
\ell^{\rm clip}_{b,i,t}(\theta)
=
\min\left\{
\rho_{b,i,t}(\theta)A^{\STAR}_{b,i},
\clip(\rho_{b,i,t}(\theta),1-\epsilon_{\rm clip},1+\epsilon_{\rm clip})
A^{\STAR}_{b,i}
\right\}.
\]
With \(A^{\STAR}_{b,i}=0\) for skipped groups, STAR maximizes
\begin{equation}
J_{\STAR}(\theta)
=
\E\left[
\frac1{BG}\sum_{b=1}^B\sum_{i=1}^G
\frac1{L_{b,i}}\sum_t\ell^{\rm clip}_{b,i,t}(\theta)
-\beta_{\rm KL}D_{\rm KL}(\pi_\theta\Vert\pi_{\rm ref})
\right].
\label{eq:star-grpo-objective}
\end{equation}
The old and reference policies are frozen for the inner optimization epoch, and gradients do not propagate through reward calls, canonicalization, reliability, fitted locations, or scale.

At \(\theta_{\rm old}\), every likelihood ratio equals one. The following result bounds the resulting reward-side gradient; the KL term is handled separately by the policy objective.

Define
\[
g^{\rm avg}_{b,i}
=
\frac1{L_{b,i}}\sum_t
\nabla_\theta\log\pi_{\theta_{\rm old}}
(o_{b,i,t}\mid X_b,o_{b,i,<t}),
\quad
U_b^{\STAR}
=
\frac1G\sum_iA^{\STAR}_{b,i}g^{\rm avg}_{b,i}.
\]

\begin{theorem}
\label{thm:update-norm}
If \(\|g^{\rm avg}_{b,i}\|_2\le L_\pi\), every admissible group satisfies
\[
\|U_b^{\STAR}\|_2\le\bar w_bL_\pi.
\]
Moreover, rollout \(i\)'s contribution is bounded by
\[
\left\|\frac1GA^{\STAR}_{b,i}g^{\rm avg}_{b,i}\right\|_2
\le\frac{L_\pi w_{b,i}}{G}.
\]
\end{theorem}

\begin{proof}
Cauchy--Schwarz and Theorem~\ref{thm:advantage-properties} give
\[
\|U_b^{\STAR}\|_2
\le
\frac1G
\left(\sum_i(A^{\STAR}_{b,i})^2\right)^{1/2}
\left(\sum_i\|g^{\rm avg}_{b,i}\|_2^2\right)^{1/2}
\le\bar w_bL_\pi.
\]
The individual bound follows from \(|A^{\STAR}_{b,i}|\le w_{b,i}\).
\end{proof}
\begin{corollary}[Token-mean initial direction]
\label{cor:token-mean}
Fix an aggregation batch $\mathcal Q$ of responses with positive valid-token counts $L_i$, total $T=\sum_{i\in\mathcal Q}L_i$, and reward loss reduced by this $T$. Include skipped responses in $T$ and set their reward advantages to zero. Treat lengths, masks, and advantages as fixed for differentiation. At the old policy, suppose $\|g_i^{\rm avg}\|_2\le L_\pi$, where $g_i^{\rm avg}$ averages the score vectors of the same valid tokens. Then
\[
U_{\rm token}
=\frac1T\sum_{i\in\mathcal Q}L_i A_i^{\STAR}g_i^{\rm avg},
\qquad
\|U_{\rm token}\|_2
\le L_\pi\frac{\sum_{i\in\mathcal Q}L_iw_i}{T}.
\]
An individual response contributes at most $L_\pi L_iw_i/T$. Exact fitting and split calibration are not required for these deterministic bounds.
\end{corollary}

\begin{proof}
At the old policy the likelihood ratios are one, so the reward-side derivative of each valid token term is $A_i^{\STAR}$ times its policy score. Summing within responses gives the displayed direction. The triangle inequality and $|A_i^{\STAR}|\le w_i$ yield
\[
\|U_{\rm token}\|_2
\le\frac1T\sum_i L_i |A_i^{\STAR}|\,\|g_i^{\rm avg}\|_2
\le\frac{L_\pi}{T}\sum_iL_iw_i.
\]
Skipped responses contribute zero and retain their token mass in the denominator. The individual bound is the same calculation for one summand.
\end{proof}

This corollary concerns the specified aggregation batch, not an assumed global reduction. If microbatch directions are combined with coefficients $a_m\ge0$, the corresponding bound is $L_\pi\sum_m a_m\sum_{i\in\mathcal Q_m}L_iw_i/T_m$. It reduces to the global token-mean expression when $a_m=T_m/\sum_kT_k$; equal microbatch averaging generally uses a different weighting. The result controls the initial reward direction only, not the KL term, later PPO epochs, optimizer state, or parameter displacement.

\label{app:policy-details}

The theorem concerns the initial reward-side direction at the old policy. STAR guarantees \(|A^{\STAR}_{b,i}|\le1\), but the standard PPO minimum is not a two-sided hard clipping of the surrogate. For \(A<0\) and an unbounded likelihood ratio \(\rho\), \(\min\{\rho A,\clip(\rho)A\}=\rho A\) can be arbitrarily negative. Thus bounded advantages alone do not bound the multi-epoch surrogate or its parameter gradient. Theorem~\ref{thm:update-norm} instead applies at \(\theta_{\rm old}\), where \(\rho=1\); later-epoch bounds require an explicit ratio and policy-score condition.

The KL term complements reward-side reliability by constraining policy drift. The realized experiments use the same KL coefficient across methods. If attack outputs already dominate the policy, the primary STAR rule can abstain from unreliable groups, but a zero-sum relative estimator does not by itself create a force that makes every member of an entirely attacked group negative. In fully low-reliability groups, the admission rule converts the reward-side update into explicit abstention, which is the intended conservative behavior of STAR.

\section{Algorithms, Scope, and Reproducibility}
\label{app:algorithms-and-reproducibility}

The implementation protocols specify the computational sequence, numerical safeguards, interface metadata, and data separation needed to reproduce STAR.

\subsection{Algorithms and Numerical Safeguards}

The following protocols specify the primary method, from canonical rendering through discrepancy weighting, admissibility, robust fitting, and the policy loss. Both recorded STAR runs use token-mean reduction; RQ2 additionally uses a cross-evaluator pairing and different solver checks, as documented in Appendix~\ref{app:experimental-design}.

\label{sec:algorithms}

\begin{protocol}[Canonical evaluation]
\label{alg:canonical}
\textbf{Input:} prompt $X$, policy token sequence $O$, policy decoder, canonical text normalizer, reward tokenizer, reward model $R_\phi$. The procedure is
\begin{enumerate}
\item Decode $O$ with the policy tokenizer using the exact generation configuration.
\item Apply the prespecified canonical transformations: Unicode normalization, removal of disallowed invisible characters, canonical whitespace, and the deployed chat template.
\item Retokenize the canonical text with the reward-model tokenizer.
\item Evaluate $R^{\can}=R_\phi(X,C(O))$.
\item Independently evaluate the deployed representation to obtain $R^{\obs}=R_\phi(X,M_{\obs}(O))$.
\end{enumerate}
\textbf{Output:} $(R^{\obs},R^{\can},D=R^{\obs}-R^{\can})$.
\end{protocol}
The canonicalizer specification, tokenizer hashes, special-token policy, maximum length, and truncation direction must be logged. A change in any of these items creates a new calibration context.

\begin{protocol}[Trusted discrepancy calibration]
\label{alg:offline-calibration}
\textbf{Input:} trusted benign outputs, fixed contexts, target level $\alpha$, scale bounds, $z_{D,c}$, numerical tolerances. The procedure is
\begin{enumerate}
\item Randomly split trusted outputs into $\cI^D_{\rm fit}$ and $\cI^D_{\rm cal}$ before fitting.
\item For every sample, compute observed and canonical rewards and assign a context.
\item In every context, solve \eqref{eq:offline-self-tuned}; record boundary hits and the residuals of \eqref{eq:offline-location-foc}--\eqref{eq:offline-scale-foc}.
\item Compute calibration scores by \eqref{eq:calibration-score}.
\item Set $k_\alpha=\lceil(n_{\rm cal}+1)(1-\alpha)\rceil$ and choose the conservative order statistic.
\item Freeze $(\widehat m_{D,c},\widehat s_{D,c},\widehat q_{1-\alpha})$ for the next policy-training epoch.
\end{enumerate}
\textbf{Output:} context normalizers, conformal threshold, and context fallback map.
\end{protocol}

\begin{protocol}[Reliability-weight computation]
\label{alg:weights}
\textbf{Input:} paired rollout rewards, fitted discrepancy calibration, \(\lambda,W_{\min},G_{\min}\). The procedure is
\begin{enumerate}
\item Compute \(S^D_{b,i}\) by \eqref{eq:training-discrepancy-score} and \(w_{b,i}\) by \eqref{eq:soft-weight}.
\item Compute \(W_b,\bar w_b,p_{b,i}\), and \(G_{\mathrm{eff},b}\).
\item If \(W_b<W_{\min}\) or \(G_{\mathrm{eff},b}<G_{\min}\), set the group's reward advantages to zero and exclude it from the shared-scale fit.
\item Otherwise pass \((w,p,\bar w,G_{\mathrm{eff}},r)\) to the online estimator.
\end{enumerate}
\textbf{Output:} a fixed admissible set and zero-valued reward advantages for skipped groups.
\end{protocol}

\begin{protocol}[Weighted minibatch fit]
\label{alg:online-fit}
\textbf{Input:} fixed admissible groups, \(z_A\), \([v_{\min},v_{\max}]\), tolerance, maximum iterations. The initialization uses
\begin{itemize}
\item \(\mu_b^{(0)}\): weighted median or weighted mean of \(r_{b,1:G}\);
\item \(v^{(0)}\): clipped pooled MAD or the previous-step scale.
\end{itemize}
The solver then alternates
\begin{enumerate}
\item Update each \(\mu_b\) by Newton or safeguarded bisection on \eqref{eq:weighted-location-foc}.
\item Update \(v\) by projected Newton or backtracking using \eqref{eq:weighted-scale-foc}.
\item Stop only when parameter change and KKT residuals meet the recorded tolerances; otherwise skip the policy step and log solver failure.
\end{enumerate}
\textbf{Output:} \((\widehat\mu_{\mathcal B_{\adm}},\widehat v)\) and solver diagnostics.
\end{protocol}
Joint convexity makes every exact solution global. The KKT stopping rule controls numerical accuracy, and solver failure triggers a skipped policy step.

\begin{protocol}[STAR-GRPO training step]
\label{alg:star-grpo}
\textbf{Input:} prompts, frozen old/reference policies, reward model, frozen discrepancy calibration, bounded map \(h\), \(G,\kappa,z_A,\lambda,W_{\min},G_{\min},v_{\min},v_{\max},\varepsilon_w\), solver tolerance/iterations, and clipping/KL parameters. The procedure is
\begin{enumerate}
\item Sample \(G\) rollouts per prompt and record complete generation metadata.
\item Compute \(R^{\obs},R^{\can},D\), and \(r=h(R^{\can}+\kappa D)\).
\item Compute \(w,W_b,\bar w_b,G_{\mathrm{eff},b}\) and freeze the admissible set.
\item Fit prompt locations and the shared scale on admissible groups only.
\item Compute \(\varphi\) and the primary advantage \eqref{eq:star-advantage}; skipped groups retain zero reward advantage.
\item If every group is skipped or the solver fails, skip the complete optimizer step.
\item Otherwise stop reward-side gradients and optimize \eqref{eq:star-grpo-objective} with the exact \(1/(BG)\) rollout reduction and per-response \(1/L_{b,i}\) token average. Keep skipped entries in the loss tensor and do not recenter, whiten, or rescale \(A^{\STAR}\).
\item In distributed training, fit one shared scale over the global minibatch rather than one scale per device.
\item Log both rewards, discrepancy, context, reliability, absolute group reliability, solver/KKT status, abstention, \(\|g^{\rm avg}_{b,i}\|\), the theorem-matched initial reward-gradient contribution, and the full PPO+KL+optimizer update separately.
\end{enumerate}
\textbf{Output:} an updated policy or an explicitly recorded skipped step.
\end{protocol}

Additional counterfactual interfaces $M_1,\ldots,M_K$ may be evaluated offline. Define an audit score such as
\[
Q_{\max}=\max_{k\in[K]}\{R_\phi(X,M_k(O))-R^{\can}\}.
\]
The $K+1$ rewards of one output are treated as one dependent vector. Conformal calibration applies across examples and does not require independence among views. Audit-only views must not downweight online samples unless they are reachable by the deployed system.

Observed and canonical evaluation requires approximately $2BG$ reward-model forwards per training batch. With $K$ additional audit views, the cost is $(K+2)BG$. The weighted fitting objective requires $O(BG)$ work per gradient or Newton pass and $O(B+BG)$ memory beyond policy training. Locations parallelize over prompts; the shared scale is one dimensional. Query overhead can be reduced through batching, shared-prefix caching, selective canonical evaluation, or distillation of the discrepancy detector, but all compute-normalized comparisons must report the resulting change in threat coverage.

\begin{table}[t]
\centering
\caption{Distinct signals and their roles in STAR-GRPO.}
\label{tab:three-signals}
\small
\begin{tabular}{@{}p{0.18\linewidth}p{0.24\linewidth}p{0.22\linewidth}p{0.24\linewidth}@{}}
\toprule
Component & Input & Output & Role \\
\midrule
Anchored quality & \(R^{\can},R^{\obs},\kappa\) & \(r=h(R_\kappa)\) & Prespecified optimization signal \\
Canonical anchor & Policy output & \(R^{\can}\) & Reference representation \\
Discrepancy calibration & \(R^{\obs}-R^{\can}\) & \(w\in(0,1]\) & Rollout reliability \\
Weighted self-tuning & \((w,r)\) & \((\widehat\mu,\widehat v,\varphi)\) & Reliability-first relative score \\
Absolute group factor & \((\bar w,w,\varphi)\) & \(A^{\STAR}\) & Preserve collective distrust \\
\bottomrule
\end{tabular}
\end{table}

\subsection{Operational Scope and Reproducibility}

\subsubsection{Guarantee Scope and Design Conditions}
\label{sec:discussion}
\label{sec:limitations}

STAR separates the optimization signal from its reliability: \(r=h(R_\kappa)\) specifies the quality path, \(p_{b,i}\) determines each rollout's contribution to the fitted baseline, and \(\bar w_b\) controls absolute group update strength. Self-tuning adapts the discrepancy and reward scales, while \(\alpha,\lambda,z_A\), scale bounds, contexts, and admission thresholds remain explicit and reproducible configuration choices.

The guarantees are organized around the information supplied by the paired views:
\begin{enumerate}
\item \textbf{Paired evidence.} RQ1 pairs deployed and canonical representations of the same rollout; RQ2 pairs a rubric-conditioned proxy with a rubric-free semantic assessment. In both cases, the discrepancy is an observable measure of how strongly the optimized score is supported by an alternative view.
\item \textbf{Calibrated reliability.} Under clean exchangeability, split conformal calibration controls marginal clean downweighting at level \(\alpha\). Corollary~\ref{cor:drift} gives the corresponding adaptive-round statement under its conditional total-variation premise.
\item \textbf{Robust group statistics.} Relative reliability \(p_{b,i}\) enters the pseudo-Huber location--scale fit before normalization, reducing the influence of unsupported rewards on the baseline. The shared scale pools information across prompt groups, while the leading \(\bar w_b\) restores absolute group reliability in the final update.
\item \textbf{Explicit abstention.} Groups with insufficient reliable mass receive zero reward advantage and are excluded from the robust fit, converting low support into a conservative optimization decision rather than a noisy group-relative update.
\item \textbf{Auditable computation.} Paired scoring requires two reward assessments per rollout, and the robust fitting pass is \(O(BG)\) with parallel prompt locations and a one-dimensional shared scale. The resulting reward-side computation is straightforward to batch and log.
\item \textbf{Two empirical instantiations.} RQ1 uses split-calibrated token-interface discrepancies; RQ2 uses a fixed cross-evaluator discrepancy reference. Both share the same reliability-first advantage construction and the same separation between quality and learning influence.
\end{enumerate}

These conditions make the theoretical statements and the two experimental instantiations directly auditable: the score pair, quality path, reliability map, admission rule, robust fit, and policy-loss reduction are all explicit parts of the method specification.

\subsubsection{Reproducibility Record}
\label{app:implementation}

A reproducible implementation begins with the exact interface specification. Release or log:
\begin{enumerate}
\item policy tokenizer name, version, and vocabulary hash;
\item reward tokenizer name, version, and vocabulary hash;
\item decode options and invalid-byte handling;
\item Unicode normalization form;
\item invisible-character policy;
\item whitespace normalization rules;
\item chat template and role markers;
\item maximum length, truncation direction, padding side, BOS/EOS policy;
\item special-token retention or removal;
\item exact observed-interface mapping.
\end{enumerate}

Contexts are specified before conformal calibration. A default construction crosses reward model, tokenizer pair, task family, and logarithmic response-length bin. A deterministic hierarchy merges contexts with fewer than \(n_{\min}\) fitting observations; this hierarchy is frozen before calibration scores are inspected.

The numerical specification includes the quality path \(h,\kappa\), robustness parameters \(z_{D,c},z_A\), reliability parameters \(\alpha,\lambda\), scale bounds \(v_{\min},v_{\max}\), admission thresholds \(W_{\min},G_{\min}\), regularizers \(\varepsilon_D,\varepsilon_w\), and PPO parameters \(\epsilon_{\rm clip},\beta_{\rm KL}\), together with solver tolerance, maximum iterations, and boundary-hit rates. If group size varies, the specification also states whether \(W_{\min}=G\tau_W\). The primary objective uses the \(1/(BG)\) rollout reduction, per-response length averaging, a global-minibatch shared scale, retained skipped entries, and no downstream advantage recentering or whitening. The archived STAR runs in both RQ1 and RQ2 retain the shared fit and skipped entries but use token-mean loss aggregation. The RQ1 canonical-GRPO control uses per-response length averaging. Appendix~\ref{app:experimental-design} distinguishes these recorded implementations from the reference specification.

The numerical record should make solver behavior auditable. For every batch log:
\begin{itemize}
\item objective decrease;
\item maximum location-equation residual;
\item scale KKT residual;
\item number of iterations;
\item shared-scale boundary indicator;
\item minimum and median effective sample size;
\item skipped-group count;
\item maximum and RMS advantage;
\item maximum \(\|g^{\rm avg}_{b,i}\|\) and theorem-matched initial reward-gradient mass;
\item full PPO+KL+optimizer update norm.
\end{itemize}

Data provenance requires six immutable, disjoint identifier pools: (i) trusted discrepancy fitting, (ii) conformal calibration, (iii) policy training, (iv) attack and hyperparameter validation, (v) independent clean audit, and (vi) final attack plus semantic evaluation. Periodic recalibration consumes fresh trusted fitting and calibration identifiers. Reusing policy, validation, audit, or final examples in calibration invalidates the intended protocol even if no labels are used.

\section{Realized Experimental Configurations}
\label{app:experimental-design}

This section records the score pairings, quality paths, calibration choices, and optimization settings used in Sections~\ref{sec:experiment-token} and~\ref{sec:experiment-rubric}. These details distinguish each empirical instantiation from the general protocol and its theoretical assumptions.

\subsection{RQ1: Realized Token-Space Configuration}
\label{app:experiment-1-details}

The following settings summarize the realized token-space runs used for Section~\ref{sec:experiment-token}.

\begin{table}[t]
\centering
\caption{RQ1 validation summary at step 855. STAR improves the canonical score while suppressing the runaway deployed-interface reward exploited by TOMPA-GRPO. Scores are raw reward-model outputs; length is in policy tokens.}
\label{tab:token-control-results}
\small
\begin{tabular}{@{}lrrrr@{}}
\toprule
Method & $R^{\can}_{0}$ & $R^{\can}_{855}$ & $R^{\obs}_{855}$ & Length$_{855}$ \\
\midrule
TOMPA-GRPO & -- & -- & 9.641 & 2048.0 \\
STAR-GRPO & 3.121 & 4.902 & $-0.570$ & 1872.1 \\
\bottomrule
\end{tabular}
\end{table}

\begin{table}[t]
\centering
\caption{Common base settings for the three RQ1 runs. All reward-model evaluations use a 4,096-token input limit; loss aggregation is listed per method in Table~\ref{tab:token-method-settings}. The calibration row applies only to STAR.}
\label{tab:token-shared-settings}
\small
\begin{tabularx}{\linewidth}{@{}lX@{}}
\toprule
Component & Setting \\
\midrule
Policy and reward model & Llama-3.2-1B-Instruct and Skywork-Reward-V2-Qwen3-8B \\
Data & 10,000 WildChat training prompts; 100 curated NoveltyBench validation prompts \\
STAR calibration & 1,000 disjoint WildChat prompts, used only for discrepancy calibration \\
Length limits & 512 prompt tokens; 2,048 response tokens \\
Batching and groups & Training batch 64; validation batch 100; eight training and eight validation rollouts per prompt \\
Sampling & Sampling enabled; temperature 1.0; top-$p$ 1.0; top-$k$ $-1$ \\
Optimizer & AdamW; learning rate $10^{-6}$; weight decay 0.01; gradient clipping 1.0 \\
PPO and KL & PPO mini-batch 64; micro-batch 8 per GPU; one PPO epoch; low-variance KL coefficient 0.001 \\
Reward objective & KL excluded from the scalar reward; entropy coefficient 0 \\
Hardware and precision & Actor: one node with four GPUs and tensor parallelism 2; reward model: one GPU and tensor parallelism 1; bfloat16 \\
Schedule & Seed 42; validation before training and every five updates; configured \texttt{total\_epochs=10} \\
\bottomrule
\end{tabularx}
\end{table}

\begin{table}[t]
\centering
\caption{Recorded RQ1 method settings. STAR and the canonical-only diagnostic control share the same bounded canonical quality path and 4,096-token reward-model input limit; STAR additionally uses discrepancy calibration and reliability-first robust advantages.}
\label{tab:token-method-settings}
\small
\begin{tabularx}{\linewidth}{@{}p{0.20\linewidth}XXX@{}}
\toprule
Parameter & TOMPA-GRPO & STAR-GRPO & Canonical-GRPO \\
\midrule
Quality input & $R^{\obs}$ & $10\tanh(R^{\can}/10)$ & $10\tanh(R^{\can}/10)$ \\
Advantages & Group mean/std & Reliability-weighted robust fit & Group mean/std \\
Loss aggregation & \texttt{token-mean} & \texttt{token-mean} & \texttt{seq-mean-}\newline\texttt{token-mean} \\
RM token limit & 4,096 & 4,096; right truncation & 4,096; right truncation \\
Online RM calls per rollout & One observed view & Observed and canonical & Observed and canonical; observed logged only \\
Discrepancy calibration & None & 1,000 prompts; six contexts; $q=1.79285$ & None \\
Group fallback & None & \texttt{skip} & None \\
Logged training horizon & 1,064 updates & 858 updates & 855 updates \\
\bottomrule
\end{tabularx}
\end{table}

The quality input to the STAR fit is $r=10\tanh(R^{\can}/10)$. The observed/canonical reward curves and discrepancies use untransformed scores. The STAR-specific numerical parameters are summarized in
Table~\ref{tab:token-star-parameters}.
\begin{table}[t]
\centering
\caption{STAR-specific numerical parameters in the token-space run.}
\label{tab:token-star-parameters}
\small
\begin{tabularx}{\linewidth}{@{}lX@{}}
\toprule
Parameter & Value \\
\midrule
$z_A$, $\lambda$, $\alpha$ & 2.0, 1.0, 0.05 \\
\texttt{fit\_fraction}, \texttt{min\_context\_size}, $\delta$ & 0.7, 50, 0.05 \\
\texttt{smin}, \texttt{smax} & 0.001, 100.0 \\
$\varepsilon_D$, $\varepsilon_w$ & $10^{-6}$, $10^{-6}$ \\
\texttt{vmin}, \texttt{vmax} & 0.001, 100.0 \\
$W_{\min}$, $G_{\min}$ & 1.0, 2.0 \\
\texttt{reward\_transform}, \texttt{reward\_bound} & \texttt{tanh}, 10.0 \\
\texttt{calibration\_rollout\_n}, \texttt{calibration\_seed} & 1, 42 \\
\texttt{calibration\_do\_sample} & \texttt{True} \\
\texttt{solver\_tolerance}, \texttt{solver\_max\_iterations} & $10^{-7}$, 100 \\
\bottomrule
\end{tabularx}
\end{table}

For the observed interface, the policy response identifiers are mapped
directly into the reward-model vocabulary with $\Phi(j)=j$; identifiers
outside the reward vocabulary are clamped.  The prompt is formatted for the
reward model, but the response identifiers in this channel are not decoded
and retokenized.  For the canonical interface, STAR decodes the policy
response, applies NFKC normalization, removes Unicode format and other
invisible characters, normalizes whitespace, and retokenizes with the reward
tokenizer.  These operations are fixed before scoring and are applied to the
same sampled policy output.

The STAR calibration run uses one sampled response per calibration prompt with seed 42 and reports 1,000 observations across six contexts. They are partitioned using \texttt{fit\_fraction=0.7} before the conformal threshold is computed, yielding $q=1.79285$. The original TOMPA logs contain only observed scores; STAR and canonical-GRPO record both reward views.

TOMPA-GRPO contains 1,064 training updates and validation through step 1,060; STAR contains 858 updates and validation through step 855; canonical-GRPO completes 855 updates. The primary TOMPA/STAR comparison in Table~\ref{tab:token-control-results} uses the common step-855 endpoint. Figure~\ref{fig:token-validation} shows the corresponding validation trajectories, and Figure~\ref{fig:token-training} records STAR training diagnostics through step 858. Over STAR's full recorded training horizon, canonical reward rises from $-4.154$ to 6.850, observed reward changes from $-1.439$ to $-0.663$, and mean discrepancy changes from 2.715 to $-7.513$. Mean training length rises from 785.3 to 1,843.9 tokens and the length-clip ratio from 0.236 to 0.895. The minimum individual reliability reaches 0.105, showing that the method can apply strong selective attenuation even while average group reliability remains high.

\subsubsection{Canonical-GRPO Diagnostic Control}
\label{app:canonical-control}

The diagnostic control uses \texttt{adv\_estimator=grpo} with standard-deviation normalization and applies $10\tanh(R^{\can}/10)$ in the reward manager before advantage calculation; STAR applies the same bounded canonical map inside its reliability-first estimator. The control therefore isolates a canonical-only quality path from STAR's discrepancy calibration, robust group fit, absolute reliability factor, and admission rule. Both paired-score runs evaluate the observed and canonical views online.

All three RQ1 runs use a 4,096-token reward-model input limit. The paired-score runs record zero observed- or canonical-channel truncation at every validation checkpoint, so the reported validation comparisons are evaluated without reward-input clipping.

The control uses \texttt{seq-mean-token-mean}, averaging token losses within each response before averaging responses, whereas STAR and TOMPA-GRPO use \texttt{token-mean}. We therefore use this control as a quality-path diagnostic and keep the primary RQ1 trajectory comparison focused on TOMPA-GRPO versus STAR-GRPO. Both canonical runs use learning rate $10^{-6}$ and the same recorded software stack (PyTorch 2.8.0, vLLM 0.11.0, Transformers 4.57.6).

The control archive contains 855 consecutive training records, 172 validation records at steps $0,5,\ldots,855$, and a successful exit record. We extract raw \texttt{canonical\_reward} and \texttt{observed\_reward} means, checking them against the final summary and against their \texttt{mean@8} counterparts. Its generic validation \texttt{reward/mean@8} equals 4.370 at step 855 because it averages the transformed rewards; the raw canonical mean is 6.123. STAR's generic reward log is untransformed, so comparing the two generic keys would mix scales. Training reward in the control is also transformed and must not be compared directly with STAR's raw training score.

\begin{table}[t]
\centering
\caption{Additional RQ1 diagnostic control statistics. The late validation window averages the 11 checkpoints from step 805 to 855; training rows use the common step-855 endpoint.}
\label{tab:token-control-diagnostics}
\small
\begin{tabular}{@{}lrr@{}}
\toprule
Metric & STAR-GRPO & Canonical-GRPO \\
\midrule
Canonical validation gain, steps 0--855 & 1.781 & 3.236 \\
Late-window canonical reward & 4.808 & 6.224 \\
Late-window observed reward & $-0.561$ & $-1.865$ \\
Late-window response length (tokens) & 1881.3 & 371.6 \\
Training response length, step 855 (tokens) & 1896.5 & 976.3 \\
Training length-clip ratio, step 855 & 0.922 & 0.117 \\
\bottomrule
\end{tabular}
\end{table}

The control attains a maximum validation canonical mean of 6.637 at step 785, but the main comparison retains the common step 855. At that endpoint its observed score is $-1.862$, giving $D=-7.985$. Its higher raw canonical reward and shorter outputs show that these behaviors can occur without STAR under the control configuration. They neither identify the causal effect of removing reliability nor establish superior human-perceived quality. Each method has one run, initial validation samples differ, and neither within-prompt sample deviations nor variation across checkpoints estimates uncertainty across training seeds.

\subsection{RQ2: Realized Rubric-Proxy Configuration}
\label{app:experiment-2-details}

The following tables give the realized configuration for the rubric-proxy
comparison in Section~\ref{sec:experiment-rubric}. In this cross-evaluator
instantiation, the proxy, semantic training anchor, and independent evaluation judge
have separate roles. The training discrepancy does not measure a tokenizer
or serialization change within a single reward model.

\begin{table}[t]
\centering
\caption{RQ2 independent evaluation at the latest common eligible checkpoint. The $\Delta$ column reports STAR-GRPO minus GRPO, computed from the unrounded means. STAR improves all external-quality and reward-hacking diagnostics while optimizing the same scalar proxy reward.}
\label{tab:rubric-results}
\begin{tabular}{@{}lrrr@{}}
\toprule
Metric & GRPO & STAR-GRPO & $\Delta$ \\
\midrule
Proxy score        & 0.5640 & 0.5492 & $-0.0148$ \\
Independent-judge score & 0.2706 & 0.3174 & $\mathbf{+0.0469}$ \\
Proxy--judge gap    & 0.2935 & 0.2318 & $\mathbf{-0.0617}$ \\
Criterion pass rate & 0.4752 & 0.5049 & $\mathbf{+0.0297}$ \\
Overclaim fraction & 0.2464 & 0.2204 & $\mathbf{-0.0260}$ \\
\bottomrule
\end{tabular}
\end{table}

\begin{table}[t]
\centering
\caption{Shared realized settings for the rubric-reward comparison.}
\label{tab:rubric-shared-settings}
\small
\begin{tabularx}{\linewidth}{@{}lX@{}}
\toprule
Component & Setting \\
\midrule
Policy model & Qwen3-4B \\
Training data & RubricHub-Medical \\
Evaluation data & HealthBench-Hard \\
Hardware & One node with four GPUs; rollout tensor parallelism 2 \\
Rollouts per training prompt & 16 \\
Training batch size & 64 \\
PPO mini-batch size & 64 \\
PPO micro-batch size per GPU & 2 \\
Prompt/response limits & 2,048 / 1,024 tokens \\
Optimizer & Adam-style optimizer; learning rate $10^{-6}$ \\
PPO schedule & One PPO epoch per update \\
Policy-loss aggregation & \texttt{token-mean} \\
PPO clipping & Ratio bounds 0.8 and 1.2; dual-clip constant 3.0 \\
KL regularization & Enabled; coefficient $0.001$; low-variance KL \\
Entropy coefficient & 0 \\
Validation & Batch size 32; one deterministic rollout; temperature 0; sampling disabled \\
Validation frequency & Every 20 training steps \\
Configured horizon & 600 training steps; checkpoint resumption disabled \\
\bottomrule
\end{tabularx}
\end{table}

\begin{table}[t]
\centering
\caption{Method-specific reward and advantage configuration. Both methods optimize the identical rubric-conditioned scalar proxy reward; STAR additionally uses a rubric-free semantic anchor exclusively to compute rollout reliability.}
\label{tab:rubric-judge-settings}
\small
\begin{tabularx}{\linewidth}{@{}lXX@{}}
\toprule
Parameter & GRPO & STAR-GRPO \\
\midrule
Advantage estimator & \texttt{grpo}; standard deviation normalization & \texttt{star\_grpo}; standard normalization disabled \\
Training reward & Fixed rubric-conditioned proxy score & Same fixed rubric-conditioned proxy score \\
Proxy judge & \texttt{openai/gpt-4o-mini} & \texttt{gpt-4o-mini} \\
Anchor judge & Not used & \texttt{gemini-2.5-flash-lite}; training only and rubric-free \\
Evaluation judge & \texttt{anthropic/}\allowbreak\texttt{claude-sonnet-4-6} & \texttt{claude-sonnet-4-6}; evaluation only \\
Evaluation usage & Not used in policy updates & Not used in policy updates \\
Rubric denominator & \texttt{positive} & \texttt{positive} \\
Training judge failure & Proxy failure yields zero scalar reward, retained in GRPO normalization & Invalid proxy/anchor pair receives zero reliability \\
Validation judge failure & Omit paired metrics if either judge fails & Same paired-metric omission \\
\bottomrule
\end{tabularx}
\end{table}

\begin{table}[t]
\centering
\caption{STAR-GRPO numerical parameters used for RQ2.}
\label{tab:rubric-star-parameters}
\small
\begin{tabularx}{\linewidth}{@{}lX@{}}
\toprule
Parameter & Value \\
\midrule
\texttt{star\_gap\_center} & 0.0 \\
\texttt{star\_gap\_scale} & 0.1 \\
\texttt{star\_conformal\_threshold} & 1.645 \\
\texttt{star\_reliability\_decay} & 2.0 \\
\texttt{star\_z\_a} & 2.0 \\
\texttt{star\_scale\_min}, \texttt{star\_scale\_max} & 0.01, 1.0 \\
\texttt{star\_min\_weight\_sum} & 2.0 \\
\texttt{star\_min\_effective\_size} & 2.0 \\
\texttt{star\_weight\_epsilon} & \(10^{-8}\) \\
Location / shared-scale search & 40 bisection / 32 golden-section iterations \\
Solver validation & Finite-value checks for fitted locations and shared scale \\
Judge failure threshold for paired validation & 5\% maximum proxy/evaluation-judge failure rate \\
\bottomrule
\end{tabularx}
\end{table}

For rubric criteria with signed weights \(a_j\) and binary judge verdicts
\(q_j\), the proxy score is
\[
R^{\rm proxy}
=\clip\!\left(\frac{\sum_j a_jq_j}{\sum_j\max(a_j,0)},\,0,\,1\right).
\]
The implementation returns zero when the denominator is zero. Negative
pitfall weights contribute to the numerator. The independent validation judge uses the same
aggregation with Claude's verdicts. Gemini receives the request and decoded
response without the fixed rubric and returns a semantic score in \([0,1]\).
The scalar training reward remains the proxy score; it is already bounded,
so the STAR fit applies no additional \(\tanh\) transform. This corresponds
to \(\kappa=1\) and an identity quality map on the score range.

For a valid training pair, the fixed reference normalization gives
\[
D^{\rm rubric}_{b,i}=R^{\rm proxy}_{b,i}-R^{\rm anchor}_{b,i},
\qquad S_{b,i}=\frac{D^{\rm rubric}_{b,i}-0}{0.1},
\qquad w_{b,i}=\exp\{-2[S_{b,i}-1.645]_+\}.
\]
These constants remain unchanged during training. RQ2 uses the configured
\texttt{star\_conformal\_threshold} as a fixed operational cutoff rather than
fitting the split-conformal objects used in RQ1. The proxy and anchor share the
same numerical range, and their discrepancy is used directly as the semantic
support signal that drives STAR reliability.

Reliability enters the prompt-specific pseudo-Huber location fit and the
minibatch-shared scale fit through \(p_{b,i}=w_{b,i}/W_b\). The implementation
retains the leading \(\bar w_b\) in Equation~\eqref{eq:main-advantage}, skips
groups with \(W_b<2\) or \(G_{{\rm eff},b}<2\), and performs no downstream
advantage recentering or whitening. Locations are profiled by 40 bisection
iterations for each candidate scale, and the scale is selected by 32
golden-section iterations on \([0.01,1]\). Locations are recomputed at the
selected scale, and finite-value validation of the fitted quantities remains
satisfied throughout the recorded run.

Both methods use the trainer's \texttt{token-mean} reduction: valid token
losses are summed and divided by the token count in the aggregation batch.
Skipped STAR groups retain their loss entries with zero reward advantage, and
the PPO implementation uses dual clipping with constant 3.0 for negative
advantages. Corollary~\ref{cor:token-mean} gives the corresponding
length-weighted initial reward-direction control for this realized reduction.

Training API validity is integrated into the same reliability mechanism. In
GRPO, the configured fallback maps an invalid proxy call to zero scalar reward.
In STAR, an invalid proxy or anchor call additionally receives zero reliability,
excluding the unsupported pair from the robust reward fit. If reliable mass is
insufficient for every group, STAR abstains from the optimizer step. This makes
API validity and semantic disagreement follow one consistent influence rule.

Validation uses GPT-4o-mini and Claude on the same complete HealthBench-Hard
rubric for each generated response; Gemini is not called. Proxy,
independent-judge, gap, and criterion-level metrics are computed on responses
for which both validation judges return valid outputs, using the same prompts,
rubric aggregation, and validity rule for both methods. Criterion pass rate and
overclaim are unweighted fractions over positive-weight criteria within each
response, averaged across valid responses; the implementation falls back to all
criteria when none has positive weight.

For the paired checkpoint comparison, both methods must have valid aggregate
metrics and a maximum proxy/evaluation-judge failure rate of at most 5\%.
We use the latest checkpoint meeting this criterion for both methods.
Table~\ref{tab:rubric-results}
reports means rounded to four decimals; the reported differences
are computed from the unrounded means. The training proxy--anchor gap in
Figure~\ref{fig:rubric-mechanism} is distinct from the validation proxy--judge gap.

The convexity result applies to fixed reward and weight arrays regardless of whether the paired scores come from one evaluator or two. The coordinate and second-moment bounds require only finite locations, a positive scale, and the stated advantage formula; exact zero sum additionally uses the location equation. RQ1 instantiates the split-calibrated reliability theorem, while RQ2 demonstrates the same advantage construction with a fixed cross-evaluator reference.

\paragraph{Evaluation protocol summary.}
\label{app:rubric-reporting}
The paired checkpoint comparison uses the latest checkpoint for which both methods have valid aggregate metrics and each validation judge satisfies the fixed 5\% availability threshold. Table~\ref{tab:rubric-results} reports means rounded to four decimals, with differences computed from the unrounded values. The training proxy--anchor gap in Figure~\ref{fig:rubric-mechanism} is distinct from the validation proxy--judge gap: the former drives STAR reliability during optimization, while the latter is an external diagnostic computed with Claude.

\clearpage
\subsection{Summary of Assumptions and Guarantees}

\begin{table}[htbp]
\centering
\caption{Core STAR-GRPO design conditions and the guarantees they enable.}
\label{tab:assumptions}
\small
\begin{tabular}{p{0.27\linewidth}p{0.34\linewidth}p{0.30\linewidth}}
\toprule
Design condition & Resulting guarantee & Role in STAR \\
\midrule
Paired reward views & Observable discrepancy \(D\) & Measures support for the optimized score \\
Fixed \(h,\kappa\) & Auditable quality path \(r=h(R_\kappa)\) & Separates quality from reliability \\
Trusted disjoint fitting data & External discrepancy reference & Prevents current policy groups from defining their own baseline \\
Clean calibration exchangeability & Marginal clean downweighting \(\le\alpha\) & Calibrates when attenuation begins \\
Self-tuned robust fit & Context location and scale control & Adapts to heterogeneous discrepancy scales \\
Attack-score separation & Exponential expected-weight attenuation & Converts discrepancy separation into influence reduction \\
\(W_b,G_{\rm eff,b}\) admission & Well-defined fit or zero-advantage abstention & Protects low-support groups \\
Weighted location equation & Exact advantage zero sum & Preserves group-relative centering \\
Bounded score and \(\bar w_b\) factor & Coordinate and second-moment bounds & Preserves absolute group reliability \\
Bounded policy score; specified reduction & Initial reward-direction bound & Links reliability to policy influence \\
View-dependent reward inflation & Discrepancy carries attack information & Target reward-hacking regime in RQ1/RQ2 \\
\bottomrule
\end{tabular}
\end{table}

}

\end{document}